%% file: main.tex
\documentclass[twoside]{article}

\usepackage[accepted]{aistats2023}
\usepackage{natbib}
\usepackage{graphicx}
\usepackage[dvipsnames]{xcolor}

\usepackage{subfig}
\usepackage{algorithm}
\usepackage{algcompatible}
\usepackage{amsmath}
\usepackage{placeins}
\usepackage{hyperref}       

\usepackage{amsmath}
\usepackage{mathtools}
\usepackage{amssymb}
\usepackage{amsthm}
\usepackage{bm}
\usepackage{xcolor}

\newcommand{\bx}{\mathbf{x}}
\newcommand{\btheta}{\boldsymbol{\theta}}
\newcommand{\bz}{\mathbf{z}}
\newcommand{\br}{\mathbf{r}}
\newcommand{\bA}{\mathbf{A}}
\newcommand{\R}{\mathbb{R}}
\newcommand{\N}{\mathbb{N}}
\newcommand{\calS}{\mathcal{S}}

\newcommand{\bin}{\mathcal{B}}
\newcommand{\Mod}[1]{\ \mathrm{mod}\ #1}

\newcommand{\ourmethod}{\texttt{BODi}}
\newcommand{\ourembedding}{\texttt{HED}}
\newcommand{\casmo}{\texttt{CASMOPOLITAN}}
\newcommand{\combo}{\texttt{COMBO}}

\newcommand{\cocabo}{\texttt{CoCaBO}}
\newcommand{\smac}{\texttt{SMAC}}

\newcommand{\dict}{{\bold A}}
\newcommand{\atom}{{\bold a}} 
\newcommand{\hamm}{{\boldsymbol \phi}}

\theoremstyle{plain}
\newtheorem{thm}{Theorem}
\newtheorem{prop}[thm]{Proposition}

\newtheorem{appthm}{Theorem} 
\newtheorem{appprop}[appthm]{Proposition} 

\graphicspath{{figures/}}

\begin{document}

\runningtitle{Bayesian Optimization over High-Dimensional Combinatorial Spaces via Dictionary-based Embeddings}
\runningauthor{Aryan Deshwal, Sebastian Ament, Maximilian Balandat, Eytan Bakshy, Janardhan Rao Doppa, David Eriksson}

\twocolumn[
    \aistatstitle{Bayesian Optimization over High-Dimensional Combinatorial Spaces \\ via Dictionary-based Embeddings}
    \aistatsauthor{Aryan Deshwal \And Sebastian Ament \And Maximilian Balandat \And Eytan Bakshy }
    \aistatsaddress{Washington State University \And  Meta \And Meta \And Meta}
    \aistatsauthor{Janardhan Rao Doppa \And David Eriksson}
    \aistatsaddress{Washington State University \And Meta}
]

\begin{abstract}
We consider the problem of optimizing expensive black-box functions over high-dimensional combinatorial spaces which arises in many science, engineering, and ML applications.
We use Bayesian Optimization (BO) and propose a novel surrogate modeling approach for efficiently handling a large number of binary and categorical parameters.
The key idea is to select a number of discrete structures from the input space (the dictionary) and use them to define an ordinal embedding for high-dimensional combinatorial structures.
This allows us to use existing Gaussian process models for continuous spaces.
We develop a principled approach based on binary wavelets to construct dictionaries for binary spaces,
and propose a randomized construction method that generalizes to categorical spaces.
We provide theoretical justification to support the effectiveness of the dictionary-based embeddings.
Our experiments on diverse real-world benchmarks demonstrate the effectiveness of our proposed surrogate modeling approach over state-of-the-art BO methods.
\end{abstract}

\section{INTRODUCTION}
Many real-world applications require building probabilistic models over high-dimensional discrete and mixed (involving both discrete and continuous parameters) input spaces using limited training data.
These models need to make accurate predictions and quantify the uncertainty for unknown inputs.
Some examples include calibration of environment models, feature selection for automated machine learning (AutoML) where the inclusion/exclusion of a given feature can be represented by a binary parameter, and microbiome analysis where the inclusion/exclusion of a microbial species is a binary parameter and environmental variables correspond to continuous parameters.

Gaussian processes (GPs)~\citep{Rasmussen2004} are well-suited for this setting.
GPs are also commonly used as surrogate models for sample-efficient optimization of expensive black-box functions over both continuous and discrete/mixed spaces~\citep{frazier2018tutorial}.
For instance, in microbiome design optimization we need to perform expensive wet lab experiments to evaluate each mixed configuration in the form of a subset of candidate microbes and environmental conditions~\citep{clark2021design}.
Other example applications include feature selection for ML models \cite{guyon2003introduction}, tuning flags of a compiler to optimize efficiency \cite{hellsten2022baco}, and tuning database configurations \cite{zhang2021facilitating}.
The key challenge in using GPs for combinatorial spaces is to define an appropriate kernel to capture the similarity between input pairs.

This paper proposes a novel Hamming embedding via dictionaries (\ourembedding{}).
This embedding allows us to leverage popular GP kernels with automatic relevance determination (ARD) for modeling high-dimensional combinatorial inputs.
Our method naturally extends to mixed inputs with both continuous and discrete variables by using a product kernel.
The key idea in our modeling approach is to select a fixed number of candidate structures from the input space, referred to as a {\em dictionary}, and to define an {\em embedding} for the input space using the Hamming distance of the inputs to elements in the dictionary.
The effectiveness of this approach critically depends on the choice of the dictionary.
Our theoretical analysis shows that the regret bound for GP bandits \citep{ucb} trained on the \ourembedding{} is a function of the cardinality of the embedded search space,
which in turn is a function of a notion of orthogonality of the dictionary.
We observe that constructing dictionaries that initially limit the collapse of the search space cardinality exhibit a high degree of modeling flexibility,
which leads to a data-driven compression of the search space and empirically faster convergence.
Motivated by these theoretical insights, we propose two methods to construct dictionaries:
1) sub-sampled binary wavelets \citep{swanson1996binary}, which optimize the orthogonality measure in power-of-two dimensions, and
2) a randomized method that generalizes to categorical inputs and allows us to design dictionaries of any size.

To evaluate the effectiveness of dictionary-based embeddings, we consider several expensive black-box optimization problems within the framework of Bayesian optimization (BO).
Applying BO to combinatorial spaces comes with unique challenges \citep{IJCAI-2021} since commonly used surrogate models often do not work well in this setting and because we cannot rely on gradient-based methods to optimize the utility function.
Examples of surrogate models that have been applied in combinatorial spaces include GPs with diffusion kernels~\citep{oh2019combinatorial}, GPs with isotropic kernels~\citep{wan2021think}, linear models~\citep{baptista2018bayesian}, and random forests~\citep{smac,l2s_disco}. When the dictionary-based embeddings are used for Bayesian optimization, we refer to this as \textbf{BO} with \textbf{Di}ctionaries (\ourmethod{}).
Our comprehensive experimental evaluation on BO benchmarks demonstrate the efficacy of \ourmethod{} over state-of-the-art methods and provide empirical evidence that \ourmethod{}'s strong performance is due to dictionary-based surrogate model.

The key contribution of this paper is the development and evaluation of our dictionary-based modeling approach.
Our specific contributions include:
\begin{enumerate}
\setlength\itemsep{0em}
    \item A dictionary-based embedding that substantially improves the quality of GP models in high-dimensional combinatorial and mixed input spaces.
    \item Two methods of constructing the dictionary: 1) via binary wavelets, and 2) a randomized construction method that generalizes to arbitrary dimensions and categorical variables.
    \item A theoretical analysis of our approach shows that it compresses the cardinality of the input space under certain conditions, and if ARD is used, in a data-dependent fashion.
    \item The compressed cardinality leads to improved regret bounds for GP Bandits with binary inputs.
    \item A comprehensive experimental evaluation on diverse set of combinatorial and mixed BO benchmarks demonstrate the effectiveness of~\ourmethod{}. The source code is available at \url{https://github.com/aryandeshwal/BODi}.
\end{enumerate}

\section{BACKGROUND}

\paragraph{Combinatorial and mixed spaces.}
Let $\mathcal{Z}$ be a combinatorial space where each element $\bold z \in \mathcal{Z}$ is a discrete structure.
We assume $\bold z \in \mathcal{Z}$ can be represented using $d$ discrete variables $v_1, v_2, \cdots, v_d$ where each variable $v_i$ takes values from a finite candidate set $C(v_i)$.
Each variable $v_i$ takes $\tau_i \geq 2$ possible values and the cardinality of the space is $|\mathcal{Z}| = \prod_{i=1}^d \tau_i$.
In particular, for binary spaces, $C(v_i) = \{0, \, 1\}$ for all $v_i$ and $|\mathcal{Z}| = 2^d$.
If $\mathcal{X}$ is a space of continuous parameters, we call $\mathcal{X} \times \mathcal{Z}$ a \emph{mixed space}.

\paragraph{Problem definition.}
We are given a {\em high-dimensional} combinatorial space $\mathcal{Z}$, i.e., the number of discrete variables $d$ is large.
We assume we are optimizing a black-box objective function $f: \mathcal{Z} \mapsto \R$, which we can evaluate on each structure $\bold z \in \mathcal{Z}$.
For example, in feature selection for Auto ML tasks, $\bold z$ is a binary structure corresponding to a subset of features and $f(\bold z)$ is the performance of a trained ML model using the selected features.
Our goal is to find a structure $\bold z \in \mathcal{Z}$ that approximately optimizes $f$ given a small number of function evaluations.

\paragraph{Bayesian optimization.}
BO methods build a probabilistic surrogate model $\mathcal{M}$, often a GP, from the training data of past function evaluations and intelligently select the sequence of inputs for evaluation in a sample-efficient manner.
The selection of inputs is performed by maximizing an \emph{acquisition function} $\alpha$ that operates on the posterior distribution provided by the surrogate model.
One of the key challenges in using BO for high-dimensional combinatorial spaces is to build accurate surrogate models, which is the central focus of our work.

\paragraph{Gaussian processes.}
GPs are non-parametric probabilistic models that are popular due to their flexibility and excellent uncertainty quantification.
A GP is specified by a mean function and a covariance function or kernel $k : \mathbb{R} \times \mathbb{R} \to \mathbb{R}$~\citep{Rasmussen2004}.
A common choice is the RBF or squared exponential kernel, which is given by
\[
    k(\mathbf{x}, \mathbf{y}) = s^2 \exp \Bigl\{ -\tfrac{1}{2} \sum_i (x_i - y_i)^2 / \ell_i^2 \Bigr\}
\]
where $\ell_i$ for $i=1,\cdots,D$ are the lengthscales that allow for automatic relevance determination (ARD) and $s^2$ is the signal variance.

\section{RELATED WORK}

\paragraph{Discrete and mixed spaces.}
In recent years, BO over discrete and mixed spaces has received considerable attention due to its wide applicability to science, engineering, AutoML, and other domains.
A variety of surrogate models have been proposed for the low-dimensional setting, but those are typically not effective for high-dimensional spaces, as we demonstrate in our experiments.

BOCS~\citep{baptista2018bayesian} targets binary spaces and employs a second-order Bayesian linear regression surrogate model, which exhibits poor scaling in the input dimension and may not support applications where the underlying black-box function requires a more complex model.
Prior work also considers different instantiations of GP models.
\combo{} \citep{oh2019combinatorial,MerCBO} employs GPs with discrete diffusion kernels over a combinatorial graph representation of the input space.
Recently, \citet{kim2022combinatorial} proposed an approach for combinatorial spaces based on continuous embeddings.
Their approach differs from ours in that they employ a uniformly random injective mapping and need to reconstruct the discrete input after optimizing the acquisition function in the embedded space.
There is also work on using deep generative models to create a latent space and apply continuous BO methods (often referred to as ``latent space BO''):  \cite{gomez2018, tripp2020sample, UAI_ermon, kajino_molecular_2019, notin2021improving, LADDER, LOL-BO}.
In contrast, our dictionary-based embeddings are computationally efficient and leverage the inherent structure in the combinatorial space. It is a fruitful direction to explore ways to synergistically combine the benefits of latent space and dictionary-based embeddings.

There is also prior work on approaches for constructing kernels over mixed spaces with both discrete and continuous variables~\citep{ru2020bayesian,oh2021mixed,deshwal2021bayesian}.
\citet{GarridoMerchn2020DealingWC} round the input variables before passing it to a GP with a canonical kernel.
Tree-Parzen Estimators (TPEs)~\citep{tpe} are applicable to mixed spaces and consider density estimation in the input space which is potentially challenging in high-dimensional settings.
SMAC~\citep{smac} employs a random forest surrogate model.

\paragraph{High-dimensional continuous spaces.}
There is a large body of work on BO over high-dimensional continuous spaces which can be classified into the following categories:
(1) Low-dimensional structure, which may be random embeddings~\citep{wang2016bayesian, Letham2019Re, NeurIPS2022}, hashing-based approaches~\citep{nayebi2019framework}, sparsity-inducing priors~\citep{eriksson2021high}, or learned embeddings~\citep{garnett2013active}.
(2) Additive structure~\citep{kandasamy2015high, gardner2017discovering}, which assumes that the high-dimensional black-box function decomposes into a sum of low-dimensional functions.
(3) Methods that avoid selecting highly uncertain boundary points.
\citet{eriksson2019scalable} use local trust regions centered around the best solutions and these trust regions are resized based on progress.
\citet{kirschner2019adaptive} optimize the acquisition function along one-dimensional lines, and \citet{BOCK} use a cylindrical kernel to focus on the interior of the domain.

\begin{figure*}
    \centering
    \includegraphics[width=0.7\textwidth]{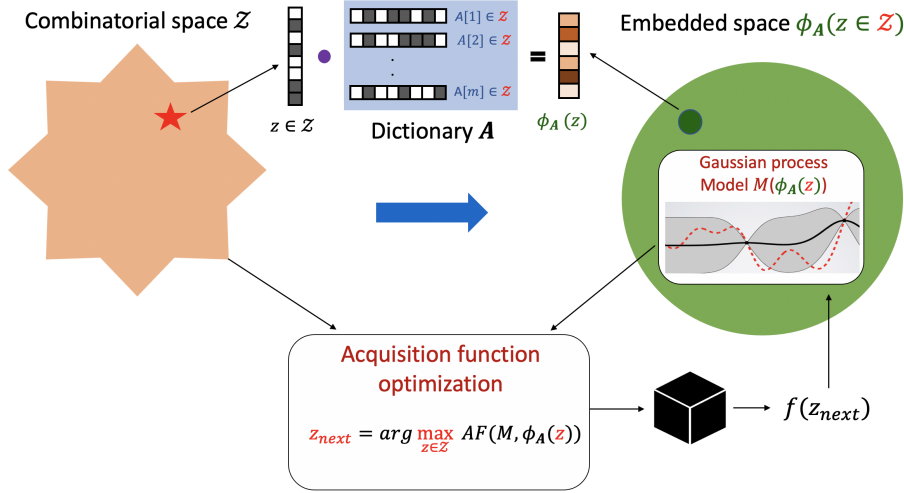}
    \caption{
       High-level overview of our \ourmethod{} algorithm for binary spaces.
       The dictionary $\dict$ contains $m$ discrete structures from the combinatorial space $\mathcal{Z}$.
        Each high-dimensional binary structure $\bz \in \mathcal{Z}$ (denoted by black and white squares) is embedded into a low-dimensional embedding $\hamm_\dict(\bz) \in \R^m$ (denoted by colored squares).
        We learn a GP surrogate model over the embedded space and perform acquisition function optimization in the original combinatorial space $\mathcal{Z}$ to select the next structure $\bz_{next}$ for function evaluation in each BO iteration.}
    \label{fig:intro_fig}
\end{figure*}


These methods, however, are specific to continuous spaces and there is little work on studying the challenges of high-dimensional combinatorial and mixed search spaces which arise in many real-world applications.
One exception is the recently proposed \casmo{} method~\citep{wan2021think}, which uses adaptive trust regions from continuous spaces~\citep{eriksson2019scalable} by replacing the standard Euclidean distance with Hamming distance for discrete (sub)spaces.
Our proposed \ourmethod{} algorithm and the associated dictionary-based kernel improve over \casmo{} in the high-dimensional setting.

\section{DICTIONARY EMBEDDINGS}
In this section, we introduce the idea of a \textbf{H}amming \textbf{e}mbedding via \textbf{d}ictionaries (\ourembedding{}),
a novel embedding for binary and categorical inputs that embeds the inputs into an ordinal feature space.
In particular, we employ a GP over the embedding $\hamm_\dict(\bold z)$ based on a dictionary $\dict$ containing $m$ discrete $d$-dimensional elements from the input space $\mathcal{Z}$.
The embedding $\hamm_\dict(\bold z)$ of size $m$ is obtained by computing the Hamming distance $h$ between $\bold z \in \mathcal{Z}$ and each element of the dictionary $\bold \atom_i \in \dict$.
That is,
\[
[\hamm_\dict(\bz)]_i = h(\atom_i, \bz).
\]

\ourembedding{} has several advantages.
First, it allows us to transform the challenging task of building models over high-dimensional discrete spaces into an application of GPs to the well-understood continuous space settings.
This subsequently allows us to perform inference of lengthscales associated with the embedding representations,
in contrast to the original categorical space where one lengthscale models the effect of a single category change.
Further, the efficient inference of lengthscales due to the embedding enables Automatic Relevance Determination (ARD) to prune away redundant dimensions effectively,
which we prove reduces the cardinality of the input space.
We show theoretically that this improves the sample-efficiency of GP bandits (UCB), commonly used for BO, and produces state-of-the-art results for BO on high-dimensional combinatorial spaces.
Further, while the core kernel is for binary spaces, it can easily be extended to mixed spaces with both continuous and discrete parameters by using a product kernel.

\begin{figure}[!ht]
    \centering
    \subfloat[Na\"ive random dictionary]{
    \includegraphics[width=0.22\textwidth]{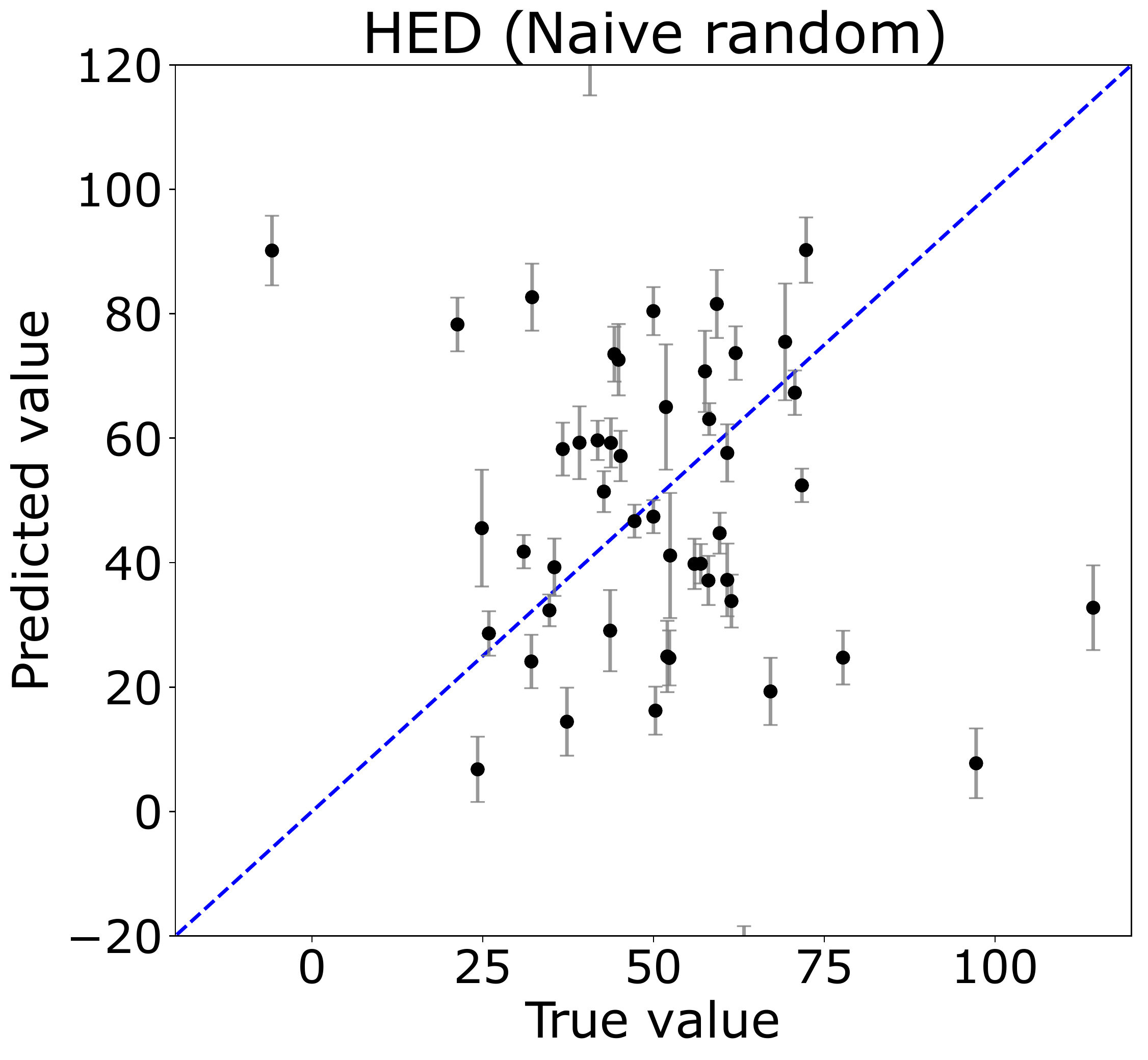}
    \label{fig:naive_random_predictions}
    }\quad
    \subfloat[Binary wavelet dictionary]{
    \includegraphics[width=0.22\textwidth]{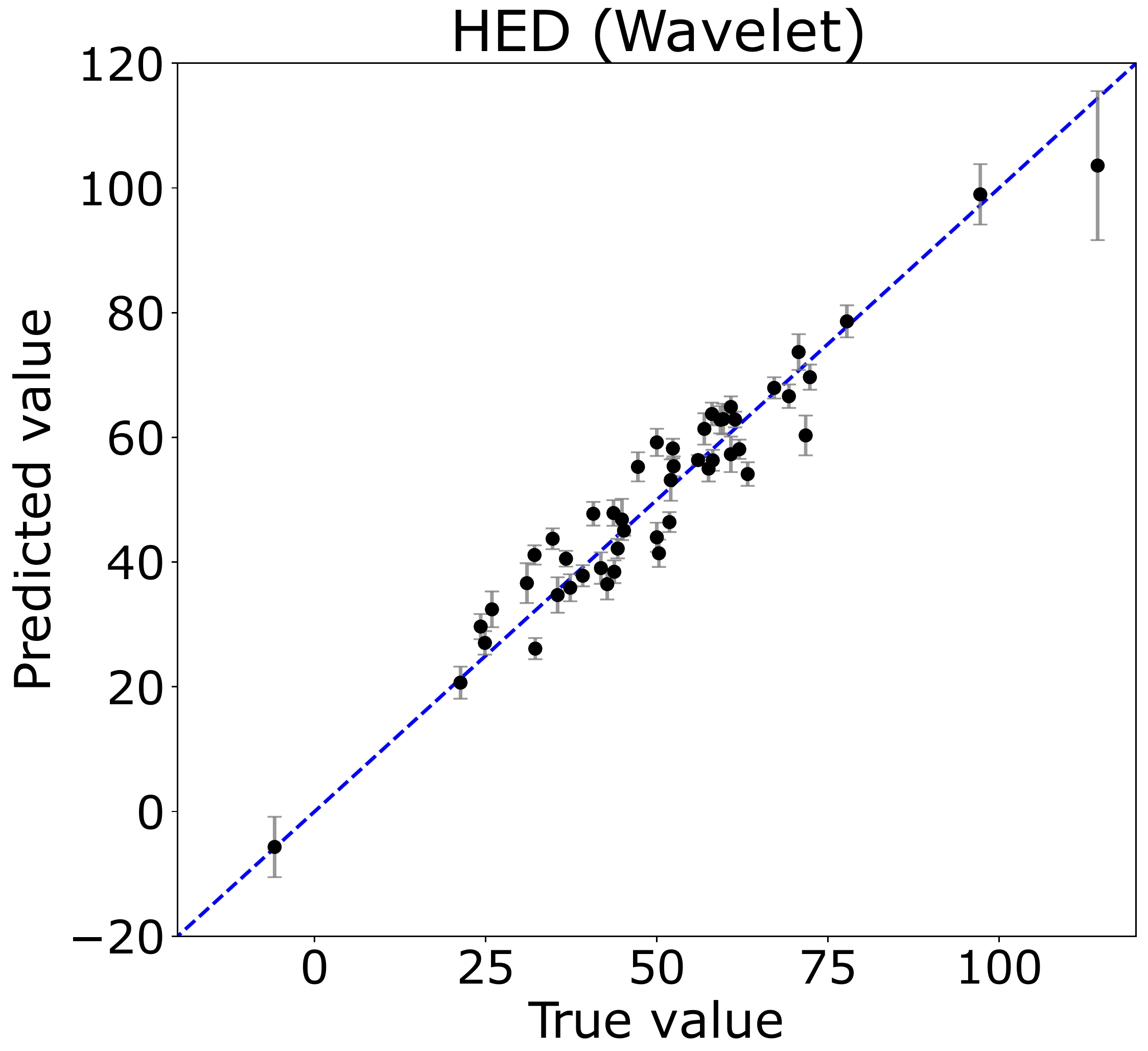}
    \label{fig:binary_wavelet_predictions}
    }\quad
    \caption{Mean predictions and associated $95$\% predictive intervals on a MaxSAT problem with 60 binary variables (see details in Sec.~\ref{sec:experiments}), comparing na\"ive random (left) and binary wavelet (right) dictionaries, using $50$ training points and predicting on $50$ test points.
}
    \label{fig:naive_wavelet_predictions}
\end{figure}

\paragraph{Dictionary construction procedure.}
The effectiveness of \ourembedding{} depends on the dictionary construction.
A na\"ive approach is to simply pick elements from the binary space uniformly at random.
However, this na\"ive approach turns out to exhibit poor predictive or BO performance on the test problems considered in this work.
For example, Fig.~\ref{fig:naive_random_predictions} illustrates the poor predictive performance of a GP using a dictionary kernel with a uniformly random binary dictionary on a MaxSAT test problem with $60$ binary variables.

Another idea is to use deterministic dictionary construction methods, such as multi-resolution {\em wavelets} \citep{mallat1989theory}, effective and well-known tools for studying real-valued signals at different scales by applying a set of orthogonal transforms to the data. In the context of binary spaces, binary wavelet transforms~\citep{swanson1996binary}
are highly related to the well-known orthogonal Hadamard matrices,
and are applied in signal processing, spectroscopy, and cryptography~\citep{hedayat1978hadamard, horadam2012hadamard}.
In contrast to the na\"ive random dictionary,
sub-sampled binary wavelet dictionaries lead to great predictive performance on the same MaxSAT problem, as shown in Fig.~\ref{fig:binary_wavelet_predictions}.

While binary wavelets constitute powerful dictionary designs for predictive and optimization problems in binary search spaces (for associated optimization results, see Fig.~\ref{fig:binwavelet_experiments}),
their construction for non powers-of-two is non-trivial, and even their existence for arbitrary dimensions is an open problem~\citep{hadamard1893resolution, baumert1962discovery, djokovic2014some}.
For this reason, we sub-sample the columns of the power-of-two dimensional binary wavelets for our experiments in non-power-of-two dimensions, see App.~\ref{app:sec:BinaryWavelet} for details.

To alleviate the difficulties around the general construction of binary wavelets, and to generalize our method to categorical spaces, we propose a randomized procedure that produces dictionary rows with a large range of sparsity levels.
We refer to this randomized procedure as ``diverse random.''

Algorithm \ref{alg:ps_dict2main} provides pseudo-code for constructing diverse random dictionaries defined over binary input spaces $\{0, 1\}^d$.
The key principle of this construction procedure is to diversify the dictionary rows by generating binary vectors determined by different bias parameters ($\theta$) of the Bernoulli distribution, unlike the na\"ive random where $\theta$ is always $1/2$.
Therefore, the rows of the na\"ive random dictionaries tend to have close to $d / 2$ non-zeros as $d$ grows, whereas the diverse random dictionaries exhibit a large range of sparsity levels due to varying $\theta$.
This algorithm can easily be generalized to inputs with categorical variables of different sizes, see App.~\ref{sec:categoricaldictionary} for details.
To summarize, the diverse random dictionaries can be constructed for arbitrary dimensions, extends naturally to categorical inputs, and as we will show later exhibits strong optimization performance on a wide range of benchmark problems.

\begin{algorithm}[!ht]
    \caption{Dictionary design for binary input space  $\{0,1\}^d$ with diversely sparse rows
    }
    \textbf{requires}: dictionary size $m$
    \begin{algorithmic}[1]
        \STATE Dictionary $\dict \leftarrow $ empty
        \FOR{$i$=$1, 2, \ldots, m$}
            \STATE $\atom_i \leftarrow $ empty
            \STATE Sample Bernoulli parameter $\theta \sim \text{Uniform}(0, 1)$
            \FOR{$j$=$1, 2, \ldots, d$}
                \STATE Sample binary number $a \sim \text{Bernoulli}(\theta)$
                \STATE $\atom_i \leftarrow \atom_i \cup a$
            \ENDFOR
            \STATE Add $\atom_i$ to dictionary: $\dict \leftarrow \dict \cup \atom_i $
        \ENDFOR
        \STATE \textbf{return} the dictionary $\dict$ of size $m \times d$
    \end{algorithmic}
    {\label{alg:ps_dict2main}}
\end{algorithm}

\paragraph{Representation of mixed input spaces.}
We have focused on a purely combinatorial input spaces $\mathcal{Z}$, but can naturally extend our approach to mixed search spaces consisting of both discrete and continuous parameters.
In this setting, we aim to model an input space $\mathcal{X} \times \mathcal{Z}$ where $\mathcal{X}$ is the domain of the continuous parameters.
To extend our approach to this mixed inputs setting, we use a product kernel leveraging the \ourembedding{} embedding for discrete parameters and a standard, e.g., Mat\'ern-$5/2$ kernel with ARD for the continuous parameters.

\section{\ourmethod{}: BAYESIAN OPTIMIZATION WITH DICTIONARY EMBEDDINGS}
\label{sec:dibo}

Our proposed \ourmethod{} method is a straightforward instantiation of the generic BO framework.
We use a GP with a standard Mat\'ern-$5/2$ kernel with ARD on the \ourembedding{} embedding as the surrogate model,
and we adopt the commonly used Expected Improvement (EI) acquisition function for single-objective problems.
In our setting, EI takes as inputs the surrogate model $\mathcal{M}$ and the embedding $\hamm_\dict(\bold z)$ to score the utility of evaluating the structure $\bold z \in \mathcal{Z}$.
In order to optimize the acquisition function over the discrete space $\mathcal{Z}$, we employ local search from randomly generated initial conditions.

Algorithm~\ref{alg:overall_bo} shows the pseudo-code of our method.
We use a small random initial training set of elements in~$\mathcal{Z}$ and their function evaluations to construct an initial surrogate model $\mathcal{M}(\hamm_\dict(\bz))$.
We generate a new dictionary~$\dict$ in each BO iteration using a randomized procedure described in Alg.~\ref{alg:ps_dict2main}, and refit the GP model using the corresponding embedding $\hamm_\dict(\bold z)$.
For each BO iteration~$j$, we select the next structure $\bz_j$ by optimizing the acquisition function.
We add $\bz_j$ and the corresponding function value $f(\bz_j)$ to the training data $D_j$ and train a new surrogate model $\mathcal{M}(\hamm_\dict(\bold z))$ using $D_j$.
We repeat these steps until the query budget is exhausted and return the best input $\bold z_{\text{best}} \in \mathcal{Z}$.

\begin{algorithm}[!ht]
    \caption{\ourmethod{} ($m$) Algorithm}
    \textbf{requires}: black-box objective $f$, discrete space $\mathcal{Z}$ with dimensionality $d$, dictionary size $m$
    \begin{algorithmic}[1]
        \STATE ${D}_0 \leftarrow$ small random initial training data
        \FOR{$j$=$1, 2, \dots$}
        \STATE Construct dictionary $\dict$ of size $m$
        \STATE Compute low-dimensional embedding $\hamm_\dict(\bold z)$ for
        \Statex  \hskip 2.0em each input structure $\bz \in D_j$ using dictionary $\dict$
        \STATE Fit a GP $\mathcal{M}$ on the embedded space $\hamm_\dict(\bold z)$
        \STATE Maximize the acquisition function in the discrete
        \Statex  \hskip 2.0em space $\mathcal{Z}$: $\bold z_{j}$ = $\arg \max_{\bold z \in \mathcal{Z}} \alpha(\mathcal{M}(\hamm_\dict(\bz)))$
        \STATE Evaluate the selected structure $\bz_{j}$ to get $f(\bz_j)$
        \STATE Aggregate training data: ${D_j} \leftarrow D_{j-1} \cup \{\bold z_j, f(\bold z_j)\}$
        \ENDFOR
        \STATE \textbf{return} $\bold z_{best}$ = $\arg \min \{f(\bold z_1), f(\bold z_2) \cdots \}$
    \end{algorithmic}
    {\label{alg:overall_bo}}
\end{algorithm}

To optimize the acquisition function over mixed search spaces, we perform alternating search over continuous and
discrete subspaces, a common approach in BO over mixed spaces~\citep{oh2021mixed,deshwal2021bayesian,wan2021think}.
We use local search for discrete parameters and gradient-based optimization 
for continuous parameters. While acquisition function optimization over discrete spaces is a challenging problem, local search with restarts has been shown to be effective in practice~\citep{oh2019combinatorial}. 

\section{THEORETICAL ANALYSIS OF \ourmethod{}}
\label{sec:theory}
In the following, we derive a surprising relationship for the Hamming embedding with an affine transformation, explaining why canonical linear embeddings (e.g. Gaussian) do not perform well.
We also provide a regret bound for BO with the dictionary kernel that crucially relies on a reduction in the cardinality -- not the dimensionality -- of the embedded search space.
Our results are stated for binary search spaces, but can be readily generalized to categorical variables using a binary encoding, e.g., one-hot encoding, or more efficiently with $\lceil \log_2(c) \rceil$ bits for $c$ categories. 

Our first proposition shows that the Hamming embedding of vectors in $\{0, 1\}^d$ is equivalent to an affine transformation of the $\{\pm 1\}$-encoding of the binary vector.
\begin{prop}[Affine Representation]
    \label{prop:affine}
    Let $\dict \in \{0, 1\}^{m\times d}$, $\bz \in \{0, 1\}^d$.
    Then
    \begin{equation}
    \label{eq:affine}
    2\hamm_\dict(\bz) = d \bold 1_m - \bar{\dict} \bar \bz,
    \end{equation}
    where $\bar a_{ij} = 2a_{ij} - 1$ and $\bar z_i = 2z_i - 1 \in \{-1, 1\}$.
    \end{prop}
    \begin{proof}
    See Appendix~\ref{app:affine}.
\end{proof}

Plugging Eq.~\eqref{eq:affine} into the embedded distance formula yields
\[
    \begin{aligned}
       2\|\hamm_\dict(\bz) - \hamm_\dict(\bz')\|_2 & = \| \bar{\dict} \bar{\br} \|_2,
    \end{aligned}
\]
where $\bar \br = \bar \bz - \bar \bz'$.
That is, the distance computation only relies on a {\it linear} projection of the difference vector $\bar \br$ of the $\{\pm 1\}$-encoding of the binary input vectors.
Furthermore, the embedding associated with the wavelet dictionary described in App.~\ref{app:sec:BinaryWavelet} is thus equivalent up to a constant shift to a sub-sampled Hadamard transform, a type of Fourier transform on Boolean fields.

Proposition~\ref{prop:affine} proves the equivalence of the dictionary-based kernel to a canonical kernel (e.g. Mat\'ern) evaluated on linearly projected input data.
Given the significant prior work on BO on subspaces~\citep{wang2016bayesian,Letham2019Re} and on properties of linear projections~\citep{kasper2017optimality}, one might assume that canonical linear embedding designs like Gaussian random matrices will perform well in our setting.
However, this is not the case, as we demonstrate in the empirical evaluation.

To understand why, first note that \ourmethod{} is effectively carrying out the optimization in the transformed search space
\[
    \label{eq:embedded_search_space}
    \calS_{\dict} = \bigl\{ \hamm_\dict(\bz) \ | \ \bz \in \{0, 1\}^d\bigr\}.
\]
While linear embeddings generally reduce the {\it dimensionality} of the search space, they do not necessarily lead to a reduction in the {\it cardinality} $|\calS_{\dict}|$, a key quantity in regret bounds for BO in finite search spaces.
Indeed, while Gaussian random projections satisfy many desirable properties, including approximate distance preservation and dimensionality reduction, our next result shows that even a one-dimensional Gaussian random projection preserves the full cardinality of the original search space almost surely.
\vspace{1ex}
\begin{prop}
    Define $\calS_{\bold a} = \{\bold a^\top \bz \ | \ \bz \in \{\pm 1\}^d\}$, and let $\atom \sim \mathcal{N}(\bold 0, \bold{I}_{d})$.
    Then $|\calS_{\atom}| = 2^d$ almost surely.
\end{prop}
\begin{proof}
    See Appendix~\ref{app:cardinality}.
\end{proof}

In contrast, our next result presents a bound on the cardinality of $\calS_\dict$ that depends on a measure of the variability $\mu_\dict$ of the dictionary rows and grows only polynomially with $d$.
\begin{prop}[Embedding Cardinality]
    \label{prop:cardinality}
    Let $\dict \in \{0, 1\}^{m \times d}$.
    Then the cardinality of the embedded search space $\calS_{\dict}$ can be bounded above by
    \[
        \left | \calS_{\dict} \right|
        \leq \left[ (\mu_{\dict}+1)(d+1-\mu_{\dict}) \right]^{\lfloor m/2 \rfloor} (d+1)^{m \Mod 2}
    \]
    where $\mu_{\dict} = \max_{i,j} \ \max(h(\atom_i, \atom_j), h(\neg \atom_i, \atom_j))$,
    and $h$ is the Hamming distance.
\end{prop}
\begin{proof}
    See Appendix~\ref{app:cardinality}.
\end{proof}

The affine representation of Prop.~\ref{prop:affine}  implies a strong similarity of $\mu_{\dict}$ to the coherence of the dictionary rows:
\[
    2\mu_\dict = d + \max_{i, j} \left | \bar \atom_i^\top \bar \atom_j \right |.
\]
The mutual coherence of dictionary columns is a central quantity in the theory of compressed sensing~\citep{tropp2004greed}.
Further, $\mu_\dict$ provides a theoretical motivation for the dictionary designs.
Indeed, the binary wavelet dictionary of App.~\ref{app:sec:BinaryWavelet} reaches the lowest possible coherence of $d/2$ in power-of-two dimensions and leads to great performance on a variety of benchmarks (see Fig.~\ref{fig:binwavelet_experiments}).
Intuitively, we want to reduce the cardinality of the search space enough to accelerate optimization, but not so much that it fails to be a useful inductive bias.
Note that $d/2 \leq \mu_\dict \leq d$ and the bound attains its maximum for $\mu_\dict = d/2$.
For example, having duplicate elements in the dictionary would imply $\mu_\dict = d$, and lead to a much larger drop in the cardinality for the same $m$ than for the wavelet dictionary of App.~\ref{app:sec:BinaryWavelet}.

We now prepare to apply the bound of Prop.~\ref{prop:cardinality} in conjunction with the seminal result of \cite{ucb} to provide an improved regret bound for \ourmethod{}.
Recall that the regret at iteration $t$ is defined by $r_t = f(\bz^*) - f(\bz_t)$, where $\bz^*$ is an optimal point and $\bz_t$ is the point chosen in the $t^{\text{th}}$ iteration.
The cumulative regret is $R_T = \sum_{t=1}^T f(\bz^*) - f(\bz_t)$ and is a key quantity in the theoretical study of BO algorithms.
Many BO methods are no-regret (i.e. $\lim_{T\to\infty} R_T/T= 0$), though the rate with which $R_T$ approaches zero varies significantly.

\cite{ucb} prove a regret bound that is sub-linear in $T$ for GP-based optimization with the upper confidence bound (UCB) acquisition function $\arg \max_{\bz} \mu_{t-1}(\bz) + \sqrt{\beta_t} \sigma_{t-1}(\bz)$, where $\mu_t$ (resp. $\sigma_t^2$) are the predictive mean (resp. variance) of the GP after $t$ iterations. 
The bound mainly depends on two quantities: (1) The information gain after $T$ iterations $\gamma_T = \log |\bold I + \sigma^{-2} \bold K_T|$, where $\bold K_T$ is the kernel matrix evaluated on the inputs $\{\bz_t\}_{t=1}^T$ that were chosen in the first $T$ iterations and $\sigma$ is the standard deviation of the observations noise.
(2) The cardinality of the search space $|\calS|$, which we bound in Prop.~\ref{prop:cardinality} for \ourmethod{}.
Notably, $\gamma_T$ depends on the kernel function and for the Mat\'ern-$\nu$ kernel in our experiments, $\gamma_T = \mathcal{O}(T^{d(d+1)/(2\nu + d(d+1))} \log T)$.
In the following, we use $\mathcal{O}^{*}$ to refer to $\mathcal{O}$ with log factors suppressed.
\vspace{1ex}
\begin{thm}
\label{thm:regret}
  Let $\dict$ have $m$ rows, $\delta \in (0, 1)$, and $\beta_t = 2 \log(|\calS_\dict|t^2\pi^2/6\delta)$.
  Then the cumulative regret associated with running UCB for a sample $f$ of a zero-mean GP with kernel function $k_\ourmethod{}(\bz, \bz') = k_{\text{base}}(\hamm_\dict(\bz), \hamm_\dict(\bz'))$, is upper-bounded by $\mathcal{O}^{*}(\sqrt{T \gamma_T m})$ with probability $1 - \delta$,
  where $\gamma_T$ is the maximum information gain of $k_{\text{base}}$.
\end{thm}
Theorem~\ref{thm:regret} exhibits a reduced dimensionality-dependent regret scaling of $\mathcal{O}^{*}(\sqrt{m})$, 
compared to $\mathcal{O}^{*}(\sqrt{d})$ for non-embedded binary inputs, as long as $m$ is not too large.
We stress that this is due to the compressed cardinality of the search space, not the reduced dimensionality of the embedding.
However, it is also important to note that not just the cardinality matters for optimization performance, since there are two main objectives that are usually at odds: (1) finding a model that is expressive enough and (2) reducing the complexity of fitting and optimizing this model.
Simply reducing the cardinality of the search space will make it easier to fit the model, but potentially less likely to accurately model the underlying black-box objective function.

Starting with a large dictionary allows the model to choose from a large number of elements and adaptively prune redundant dimensions via ARD.
In fact, our experiments confirm that larger embedding dimensions tend to improve performance and that ARD effectively prunes away the majority of embedding dimensions (see Sec.~\ref{sec:model_fits}).
The fact that the embedding values are ordinal, rather than binary, likely aids the inference of appropriate length scales.
This results in the search space cardinality reduction shown by Prop.~\ref{prop:cardinality}.

\section{EXPERIMENTS}
\label{sec:experiments}
We evaluate \ourmethod{} on wide range of challenging optimization problems for combinatorial and mixed search spaces.
We compare against several competitive baselines including \casmo{}, \combo{}, \cocabo{}, \smac{}, and random search.

\paragraph{Experimental setup.}
We use expected improvement as the acquisition function for all experiments.
However, note that our approach is agnostic to this choice and any other acquisition function can be employed, which makes it easy to extend \ourmethod{} to, e.g., multi-objective, multi-fidelity, and constrained settings.
We employ a Mat\'ern-5/2 kernel with ARD for both discrete and continuous variables.
When considering combinatorial search spaces, we optimize the acquisition function using hill-climbing local search, similarly to the approach used by \casmo{}~\citep{wan2021think}.
We follow Alg.~\ref{alg:ps_dict2main} (App.~\ref{sec:categoricaldictionary}) and $m = 128$ and the diverse random approach to construct dictionaries for all experiments.
The choice $m=128$ is investigated in an ablation study in Fig.~\ref{fig:ablation}.
Our code is built on top of the popular GPyTorch~\citep{gardner2018gpytorch} and BoTorch~\citep{balandat2020botorch} libraries.
We use the open-source implementations for all the baselines: \casmo{}~\footnote{\tiny\url{https://github.com/xingchenwan/Casmopolitan}}, \combo{}~\footnote{\tiny\url{https://github.com/QUVA-Lab/COMBO}}, \cocabo{}~\footnote{\tiny\url{https://github.com/rubinxin/CoCaBO_code}}, and \smac{}~\footnote{\tiny\url{https://github.com/automl/SMAC3}}.

\begin{figure*}[!ht]
    \centering
    \subfloat[LABS ($50$ binary parameters)]{
    \includegraphics[width=0.305\textwidth]{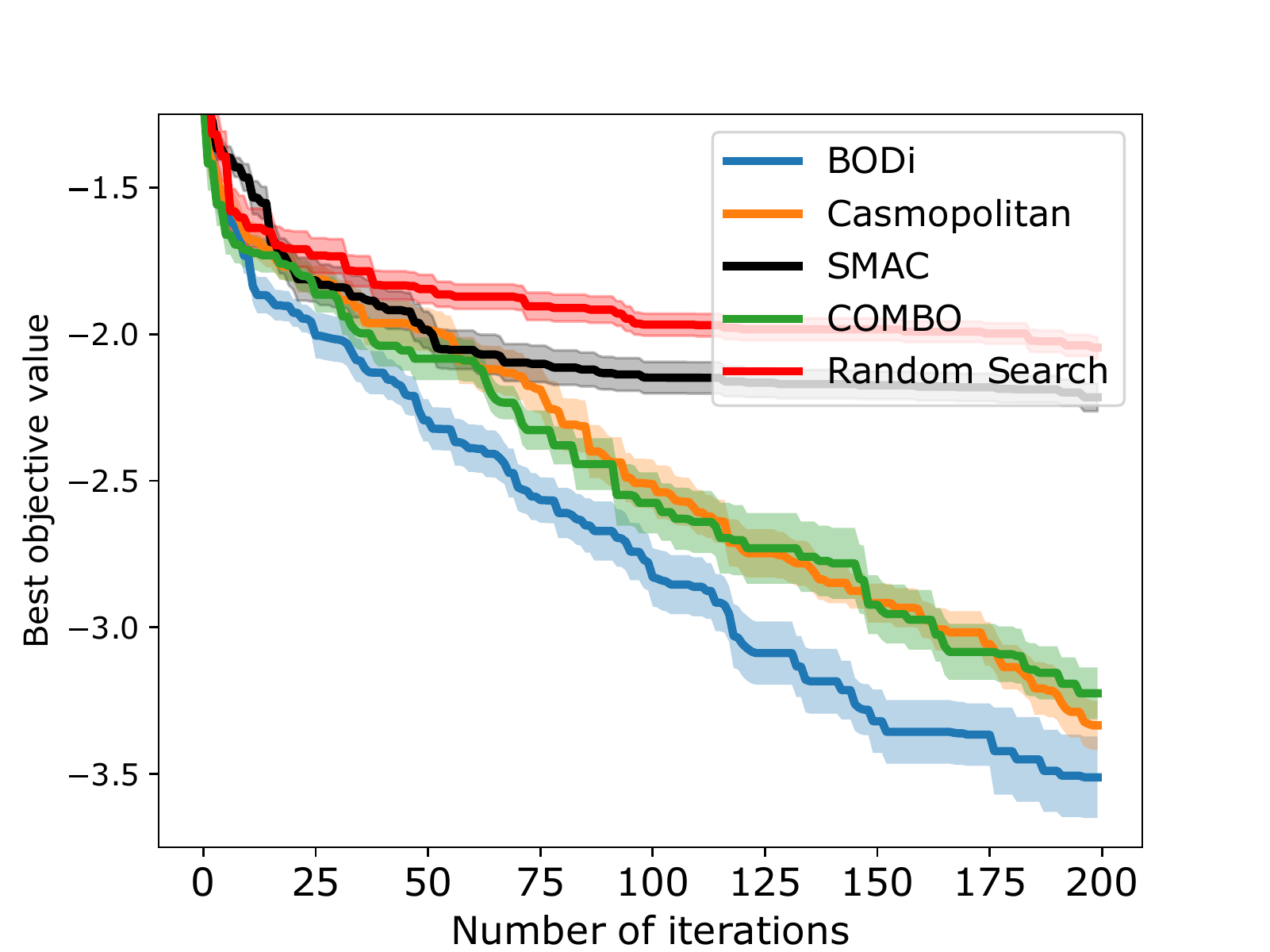}
    \label{fig:labs}
    }\quad
    \subfloat[MaxSAT ($60$ binary parameters)]{
    \includegraphics[width=0.30\textwidth]{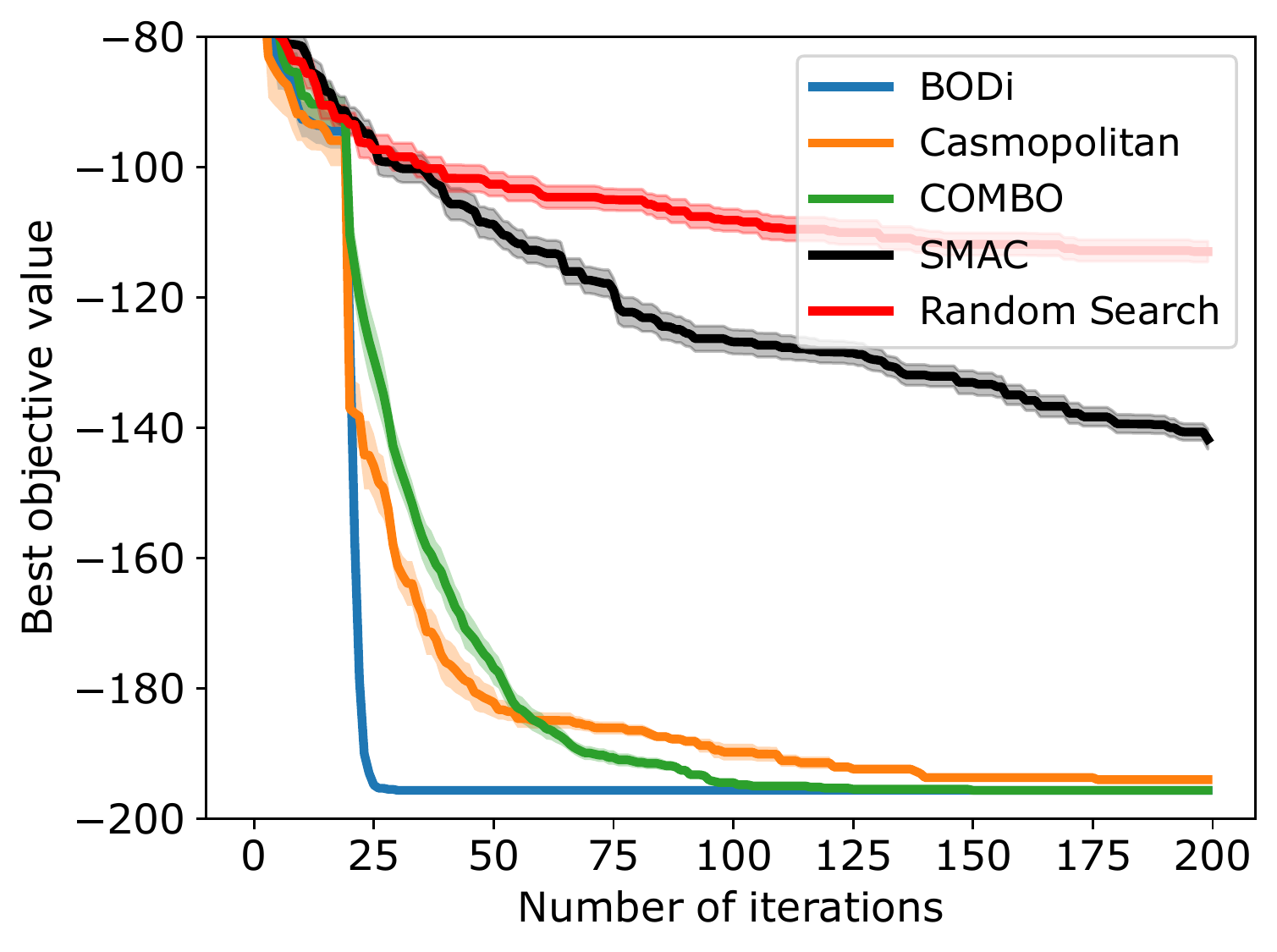}
    \label{fig:maxsat60}
    }\quad
    \subfloat[Pest Control ($25$ categorical \\ parameters with $5$ possible values)]{
    \includegraphics[width=0.285\textwidth]{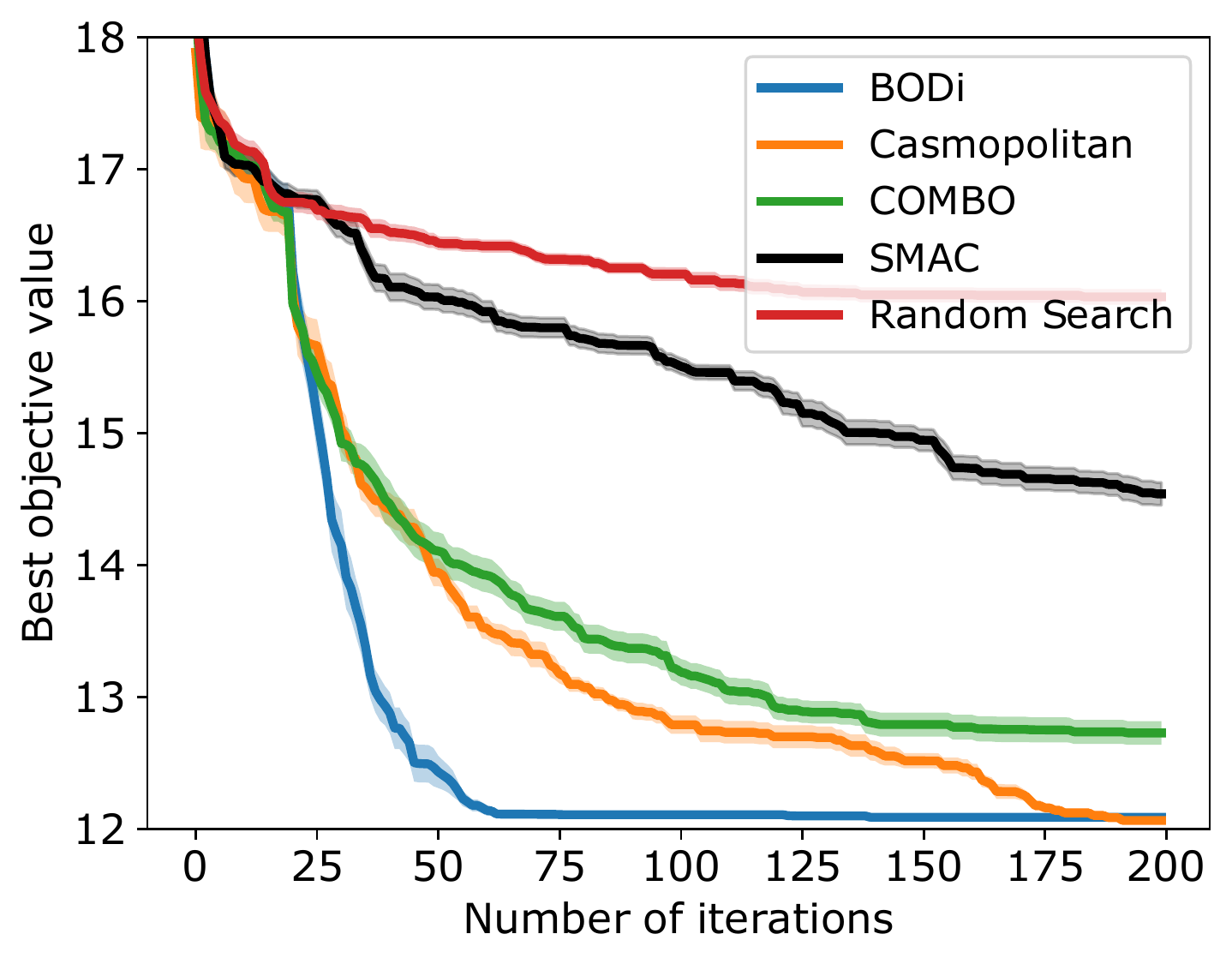}
    \label{fig:pestcontrol}
    }\quad
    \caption{We compare \ourmethod{} to \casmo{}, \combo{}, \smac{}, and random search on three high-dimensional combinatorial test problems. We find that \ourmethod{} consistently performs the best followed by \casmo{} and \combo{}.}
    \label{fig:experiments1}
\end{figure*}

\subsection{Combinatorial test problems}
\paragraph{LABS.}
The goal in the Low Auto-correlation Binary Sequences (LABS) problem is to find a binary sequence $\{1, -1\}$ of length $n$ that maximizes the  {\em Merit factor (MF)}:
\begin{align*}
    \max_{\mathbf{x} \in \{1, -1\}^n} \text{MF($\mathbf{x}$)} &= \frac{n^2}{E(\mathbf{x})} \hspace{1mm}, \\
    E(\mathbf{x}) &= \sum_{k=1}^{n-1} \left(\sum_{i=1}^{n-k} x_i x_{i+k}\right)^2
\end{align*}
This problem has diverse applications in multiple fields~\citep{LABS_statistical_physics,LABS_main}, including communications where it is used in high-precision interplanetary radar measurements of space-time curvature~\citep{LABS_communication}.
We evaluate all methods on the $50$-dimensional version of this problem.
Fig.~\ref{fig:labs} plots the negative MF and shows that \ourmethod{} finds significantly better solutions than the baselines.
While \combo{} and \casmo{} perform worse than \ourmethod{}, they find better solutions than \smac{}.
Random search performs quite poorly, indicating the importance of employing model-guided search techniques for challenging problems (the combinatorial space for LABS has $2^{50} \approx 1.2 \times 10^{15}$ configurations).
Note that \citet{LABS_main} published the optimizer $\bold x_{\text{opt}}$ of the 50-dimensional LABS problem with $\text{MF}(\bold x_{\text{opt}}) = 8.170$,
which was computed with a branch-and-bound algorithm at exponential computational cost.
We emphasize that our results here are not meant to advocate for the solution of this particular LABS problem using BO, but to serve as a comparison of the BO algorithms, which are designed to be sample efficient, on a challenging combinatorial optimization task.

\paragraph{Weighted maximum satisfiability.} The goal of this problem is to find a $60$-dimensional binary vector that maximizes the combined weights of satisfied clauses.
We use the benchmark problem \texttt{frb-frb10-6-4.wcnf}\footnote{\tiny \url{https://maxsat-evaluations.github.io/2018/index.html}}
of the Maximum Satisfiability Competition 2018\footnote{\tiny \url{http://sat2018.azurewebsites.net/competitions/}}, similar to
\citet{oh2019combinatorial} and \citet{wan2021think}.
Satisfiability problems are ubiquitous and frequently arise in many fundamental areas of computer science~\citep{biere2009handbook}.
Fig.~\ref{fig:maxsat60} shows that \ourmethod{} is quickly able to find a close-to-optimal solutions even though this combinatorial search space has as many as $2^{60} \approx 1.2 \times 10^{18}$ possible configurations.
The strong performance of \ourmethod{} on this problem is due to the superior model performance of the GP trained on the~\ourembedding{}, see Sec.~\ref{sec:model_fits}.

\paragraph{Pest control.} This problem concerns the control of pest spread in a chain of $25$ stations where a categorical choice of $5$ possible options can be made at each station to use a pesticide differing in terms of their cost and effectiveness.
This problem is challenging due to the $5^{25} \approx 3.0 \times 10^{17}$ total number of configurations.
From Fig.~\ref{fig:pestcontrol} we observe that \ourmethod{} quickly converges to a solution with objective value around $\approx 12$ and substantially outperforms the other baselines on this problem.

\begin{figure*}[!ht]
    \centering
    \subfloat[Mixed Ackley ($50$ binary parameters, \\ $3$ continuous parameters)]{
    \includegraphics[width=0.29\textwidth]{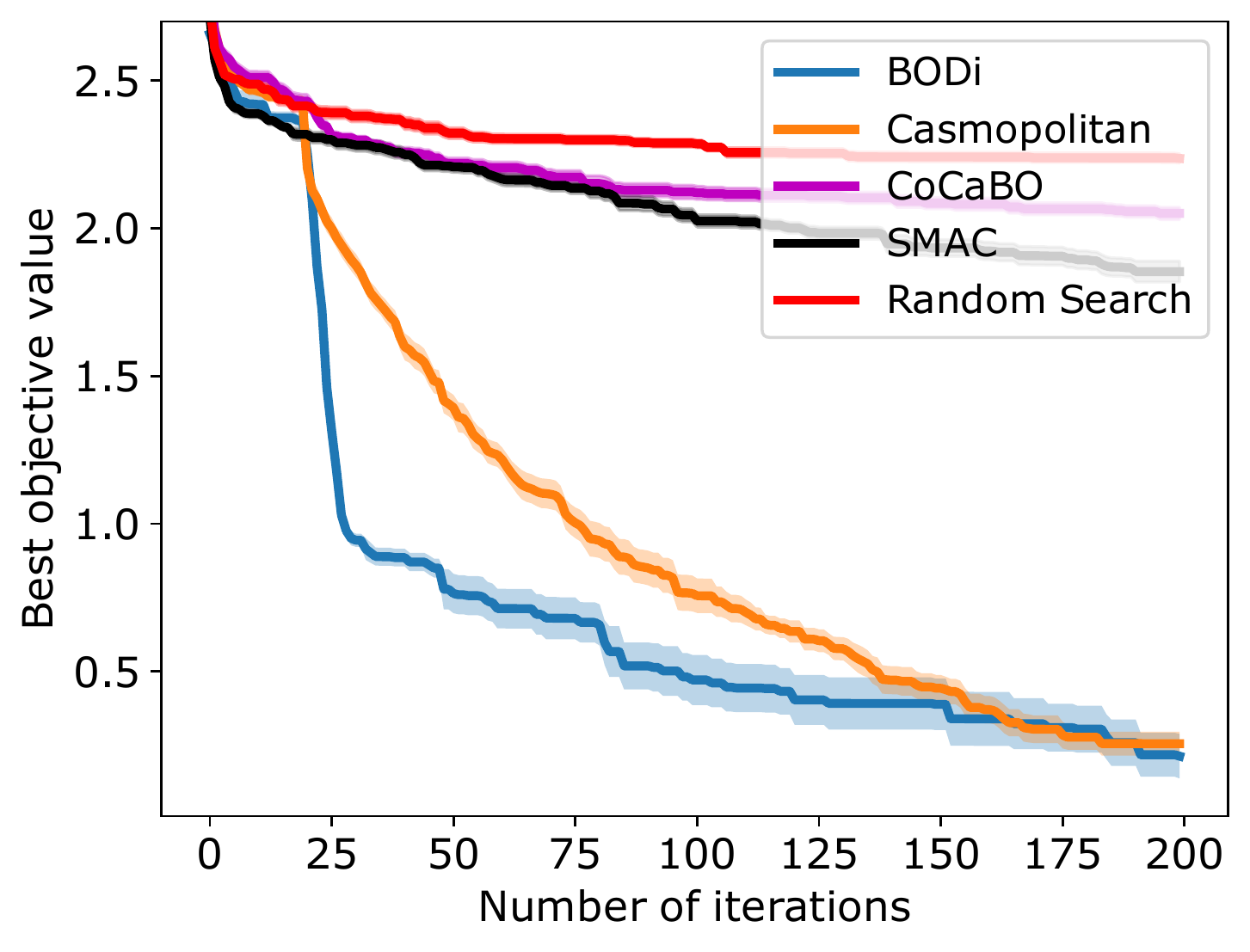}
    \label{fig:ackley53}
    }\quad
    \subfloat[SVM ($50$ binary parameters, \\ $3$ continuous parameters)]{
    \includegraphics[width=0.30\textwidth]{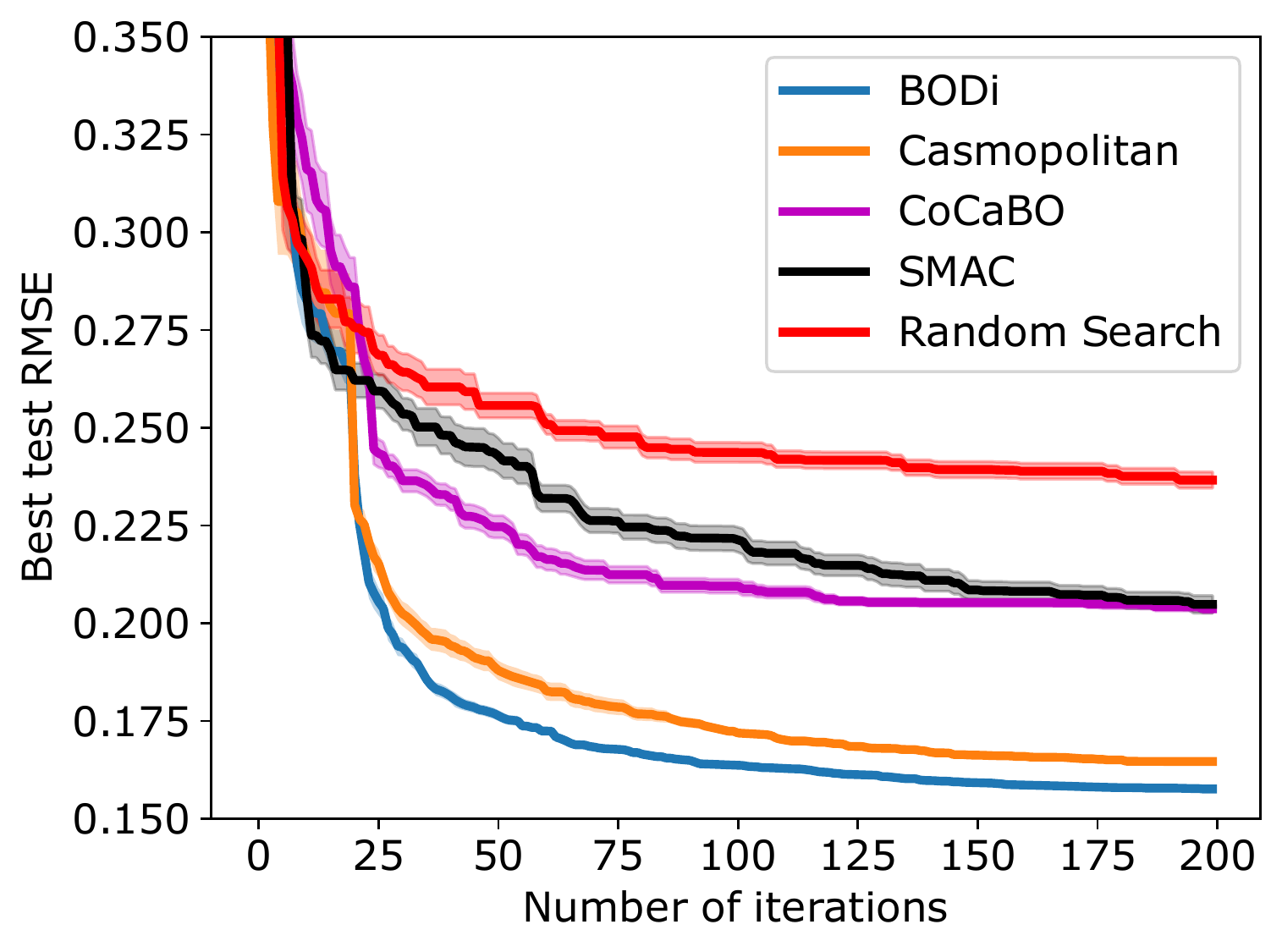}
    \label{fig:svm}
    }\quad
    \subfloat[LABS ablation ($50$ binary parameters)]{
    \includegraphics[width=0.305\textwidth]{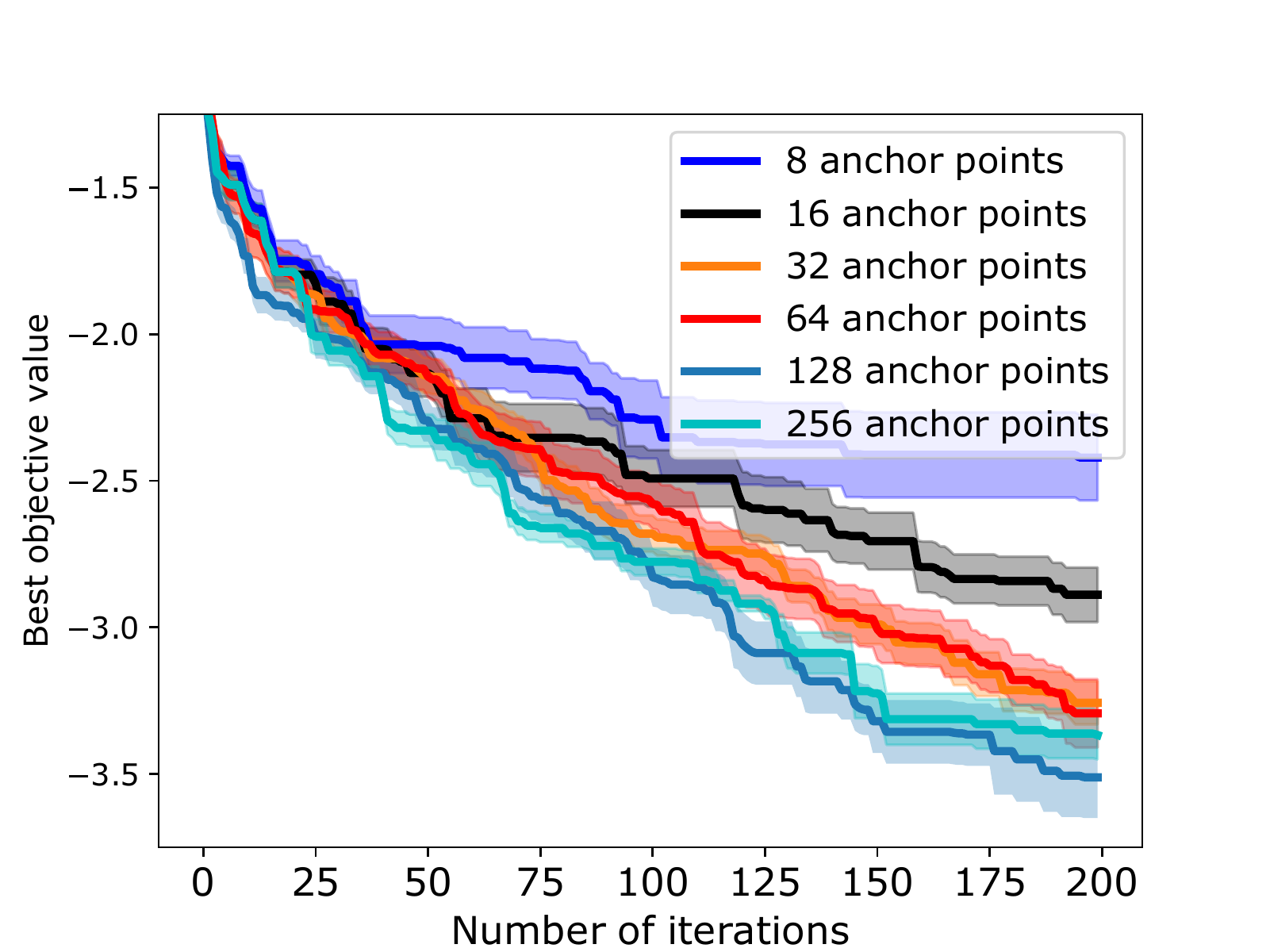}
    \label{fig:ablation}
    }\quad
    \caption{
        (Left, Middle) We compare \ourmethod{} to \casmo{}, \cocabo{}, \smac{}, and random search and two high-dimensional problems with both discrete and continuous parameters.
        \ourmethod{} converges faster than \casmo{} on the Ackley problem and performs better on the SVM problem.
        (Right) We study the sensitivity of \ourmethod{} to the size of the dictionary ($m$) and observe consistent performance as long as we do not use dictionaries with too few elements.}
    \label{fig:experiments2}
\end{figure*}

\subsection{Model performance}
\label{sec:model_fits}
To validate that a GP using the \ourembedding{} provides accurate and well-calibrated estimates relative to categorical overlap kernels (used in \casmo, \citep{wan2021think}), and the diffusion kernel (used in \combo, \citep{oh2019combinatorial}), we examine the predictive performance of these different kernels on a $60$-dimensional MaxSAT problem.
We generate $50$ training points and $50$ test points and compare the test predictions of the dictionary-based kernel with the GP relative to the overlap kernel and diffusion kernel.
The mean predictions on the test set with associated $95$\% predictive intervals are shown in Fig.~\ref{fig:model_fits}.

The \ourembedding{} with diverse random dictionary elements gives rise to an accurate model of the unknown black-box function, while overlap and diffusion kernels fail to produce accurate test predictions.
In addition, we also observe that \ourembedding{} with a Gaussian random dictionary -- computed via the affine representation of Prop.~\ref{prop:affine} -- performs poorly.
Finally, even though we use dictionaries with $m=128$ elements in Fig.~\ref{fig:model_fits}, it turns out that only $4$ of them have a lengthscale below $10$ in the fitted GP model.
This shows that ARD is able to effectively prune away the majority of dictionary elements and only use a small number of them, which leads to a tighter regret bound according to Thm.~\ref{thm:regret}.

\subsection{Mixed test problems}
\paragraph{Mixed Ackley.} We consider a mixed version of the standard Ackley problem from~\citep{wan2021think} with $50$ binary and $3$ continuous variables.
We see that \ourmethod{} makes quick progress and approaches the global optimal value of~$0$ (Fig.~\ref{fig:ackley53}).
Except for \casmo{}, all other baselines perform poorly on this problem.
Notably, the sub-sampled binary wavelet dictionary also performs particularly well on this problem, see App. Fig.~\ref{fig:binwavelet_ackley}.

\paragraph{Feature selection for SVM training.} In this problem, we consider joint feature selection and hyperparameter optimization for training a support vector machine (SVM) model on the UCI slice dataset~\citep{dua2019uci}.
We optimize over the inclusion/exclusion of~$50$ features, and additionally tune the $C$, $\epsilon$, and $\gamma$ hyperparameters of the SVM.
The goal is to find the optimal subset of features and values of the continuous hyperparameters in order to minimize the RMSE on a held-out test set.
Fig.~\ref{fig:svm} shows that \ourmethod{} performs slightly better than \casmo{} on this real-world problem.

\subsection{Ablation study}
We perform an ablation study on the sensitivity of \ourmethod{} to the number of elements of the dictionary (dictionary size).
We consider the $50$-dimensional LABS problem.
The results in Fig.~\ref{fig:ablation} show that dictionaries with $m=128$ or $m=256$ elements perform the best (albeit differences in performance are relatively small, at least for larger~$m$).
We observe that using a small dictionary (with $m=16$ or $m=32$ elements) results in inferior performance.
On the other hand, using a large number of elements increases the runtime of our method, which is why we opted for the choice of $m=128$ for all experiments.

\begin{figure}[!t]
    \centering
    \includegraphics[width=0.47\textwidth]{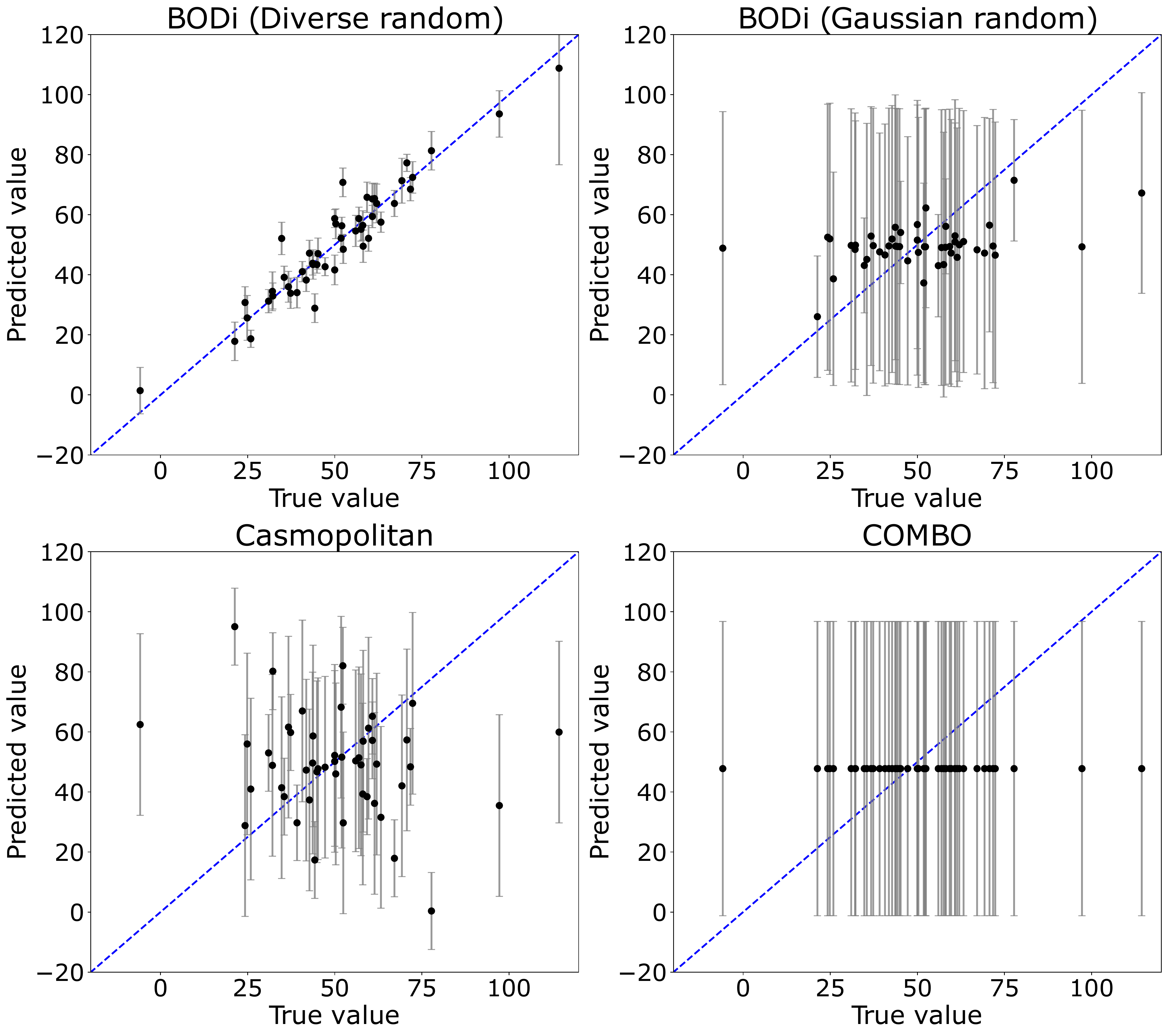}
    \caption{
    Mean predictions and associated $95$\% predictive intervals on then MaxSAT problem for \ourmethod{} with diverse random dictionary (top left), \ourmethod{} with a Gaussian random dictionary via the affine representation of Eq.~\eqref{eq:affine} (top right), \texttt{Casmopolitan} (bottom left), and \texttt{COMBO} (bottom right).
    We use $50$ training points and predict on $50$ test points.
    \ourmethod{} with the diverse random dictionary performs much better than with the Gaussian random embedding, validating our theoretical results in Sec.~\ref{sec:theory}.
    Our kernel also outperforms the isotropic kernel used by \casmo{} and the diffusion kernel used by \combo{}.}
    \label{fig:model_fits}
\end{figure}

\section{DISCUSSION}
We introduced a novel dictionary kernel for GP models, which is suitable for high-dimensional combinatorial search spaces (and can be straightforwardly extended to mixed search spaces).
While we focused on using our dictionary-based modeling approach for BO, the implications of our contributions go far beyond BO alone and are relevant for kernel-based methods more generally.
In the context of BO, our dictionary kernel is agnostic to the choice of acquisition function and can be easily applied to settings such as multi-objective and multi-fidelity optimization, and can also be combined with ideas such as trust region optimization.
\ourmethod{} showed strong performance on a diverse set of problems and outperformed several strong baselines such as \casmo{} and \combo{}.

Our work has a few limitations and raises a number of interesting questions that warrant further exploration.
While \ourmethod{} is agnostic to the choice of acquisition function, we only evaluated its performance on single-objective problems.
In addition, rather than randomly generating a diverse set of dictionary elements, we may be able to further improve the dictionary-based GP model by optimizing the dictionary as part of the model fitting procedure.
This may be particularly useful in cases where we have access to historical data that can help us discover suitable dictionaries.
Alternatively, there may be ways of generating the dictionaries in a way that is more aligned with the goal of BO, which is not to fit a globally accurate model but rather identify the location of the global optimum.
Finally, \ourmethod{} may also benefit from recently proposed methods for efficient acquisition function optimization in mixed search spaces~\citep{daulton2022pr}.

\noindent {\bf Acknowledgements.} Aryan Deshwal and Jana Doppa were supported in part by the National Science Foundation grants IIS-1845922 and OAC-1910213.

\FloatBarrier
\bibliographystyle{abbrvnat}
\bibliography{ref.bib}

\clearpage
\onecolumn

\hsize\textwidth\linewidth\hsize\toptitlebar
{\centering{\Large\bfseries Bayesian Optimization over High-Dimensional Combinatorial Spaces \\ via Dictionary-based Embeddings \\ Supplementary Materials \par}}
\bottomtitlebar

\input{supplementary}

\thispagestyle{empty}

\end{document}

%% file: supplementary.tex
\appendix

\section{Affine Representation of the Hamming Embedding}
\label{app:affine}
Our first proposition shows that the Hamming embedding of vectors in $\{0, 1\}^d$ is equivalent to an affine transformation of the $\{-1, 1\}$-encoding of the original binary vector.
\begin{appprop}[Affine Representation]
Let $\dict \in \{0, 1\}^{n\times d}$, $\bold z \in \{0, 1\}^d$.
Then
\begin{equation}
2\hamm_\dict(\bz) = d \bold 1_n - \bar{\dict} \bar \bz,
\end{equation}
where $\bar A_{ij} = 2A_{ij} - 1$ and $\bar z_i = 2z_i - 1 \in \{-1, 1\}$.
\end{appprop}
\begin{proof}
Let $\atom_i$ be the $i$th column in $\dict$, and $z_i$ the $i$th entry of $\bz$. Then
\[
\begin{aligned}
\hamm_\dict( \bold x) &= \sum_{i}^d (\neg \atom_i z_i + \atom_i \neg z_i) \\
	&= \sum_{i}^d ([\bold 1_n - \atom_i]  z_i + \atom_i  [1 - z_i]) \\
	&= \sum_{i}^d (\bold 1_n z_i - 2\atom_i z_i + \atom_i) \\
	&= \bold 1_n (\bold 1_d^\top \bz) - \dict (2\bz-1) \\
	&= [(\bold 1_n \bold 1_d^\top) (2\bz - 1) + d \bold 1_n ]/2 - \dict (2\bz-1) \\
	&= [d \bold 1_n - (2\dict - \bold 1_{n, d}) (2\bz - \bold 1_d)] / 2 \\
	&= (d \bold 1_n - \bar{\dict} \bar \bz)/ 2. \\
\end{aligned}
\]
Multiplying both sides by two finishes the proof.
\end{proof}

Plugging the affine representation into the embedded distance formula yields
\[
\begin{aligned}
   2\|\hamm_\dict(\bz) - \hamm_\dict(\bz')\|
    & = \| (d \bold 1_n - \bar{\dict} \bar \bz) - (d \bold 1_n - \bar{\dict} \bar \bx')\| \\
    & = \| \bar{\dict} \bar \bz - \bar{\dict} \bar \bz'\| \\
    & = \| \bar{\dict} \bar{\br} \|, \\
\end{aligned}
\]
where $\bar \br = \bar \bz - \bar \bz'$.
That is, the distance computation only relies on a {\it linear} projection of the difference vector $\bar \br$ of the $\{-1, 1\}$-encoding of the centered input vectors.
As a further consequence, if the wavelet dictionary of Section~\ref{app:sec:BinaryWavelet} is chosen, the embedding is a sub-sampled Hadamard transform up to a constant shift,
which we could implement by means of the Fast Hadamard Transform in $d \log d$ time.

Another consequence of the affine representation is the that the Hamming distance $h$ can first be written as the Euclidean distance of the shifted inputs $\bar \bz$.
Further, we can use the fact that $\|\bar \bz\|_2^2 = d$ to write
\[
\begin{aligned}
2h(\bz, \bz') &= d - \bar \bz^\top \bar \bz'
    = (\|\bar \bz\|_2^2  + \|\bar \bz'\|_2^2) / 2  - \bar \bz^\top \bar \bz'
    = \|\bar \bz - \bar \bz'\|_2^2 / 2.
\end{aligned}
\]
And thus the exponentiated negative Hamming distance can be seen as an RBF kernel:
\begin{equation}
\label{eq:hamming_rbf}
\exp(-2h(\bz, \bz')) = \exp(-\|\bar \bz - \bar \bz'\|_2^2 / 2).
\end{equation}

\section{Search Space Cardinality Reduction}
\label{app:cardinality}
This section derives a bound on the cardinality of the space of embedded inputs $\hamm_\dict(\bz)$.
Using the dictionary kernel is equivalent to applying a canonical kernel
to the transformed search space $\calS = \{ \hamm_\dict( \bz) \ | \ \bz \in \{0, 1\}^d\}$.
Therefore, generic convergence and regret bounds for finite search spaces apply.
However, while generic linear embeddings generally reduce the dimensionality of the search space,
they do not necessarily lead to a reduction in the cardinality $|\calS|$,
a key quantity in regret bounds for Bayesian optimization in finite search spaces.
Indeed, the next result shows that even for a one-dimensional Gaussian random projection, the full cardinality is preserved.
\begin{appprop}
    Define $\calS_{\bold a} = \{\bold a^\top \bz \ | \ \bz \in \{\pm 1\}^d\}$, and let
    $\atom \sim \mathcal{N}(\bold 0, \bold{I}_{d})$.
    Then $|\calS_{\atom}| = 2^d$ almost surely.
\end{appprop}
\begin{proof}
    Given $\bz, \bz'\in \{-1, 1\}^d$,
    suppose $\bz \neq \bz'$ and $\atom^\top \bz = \atom^\top \bz'$.
    Therefore, $\atom^\top (\bz - \bz') = 0$.
    Since $\bz - \bz' \neq 0$, this can only hold if $\atom \ \bot \ (\bz - \bz')$.
    But $\{\atom \ | \ \atom \ \bot \ (\bz - \bz')\}$ is $(d-1)$-dimensional,
    and therefore a nullset under the Gaussian measure in $d$ dimensions~\citep{rudin1974real}.
    Therefore, $\atom^\top (\bz - \bz') \neq 0$ almost surely.
    Since the set $\{-1, 1\}^d$ has finite cardinality $2^d$,
    and by the subaddativity of any probability measure $\mu$,
    \[
    \mu \left( \bigcup_{\bz, \bz' \in \{-1, 1\}^d} \{\bold a \ | \ \bold a \ \bot \ (\bz - \bz') = 0 \} \right )
    \leq
    \sum_{\bz, \bz' \in \{-1, 1\}^d} \mu \left(\{\bold a \ | \ \bold a \ \bot \ (\bz - \bz') = 0 \} \right )
    = 0.
    \]
    Thus, all distinct $\bz \in \{-1, 1\}^d$ map to distinct values $\bold a^\top \bz$ almost surely,
    so $|\calS_{\bold a}| = |\{-1, 1\}|^d = 2^d$.
\end{proof}

The following proposition sheds light on the implied cardinality of the embedded search space
as a function of the number of embedding dimensions $n$, the input dimensionality $d$,
and a measure of the variability of the dictionary rows.

\begin{appprop}[Embedding Cardinality]
\label{apppropr:cardinality}
    Let $\dict \in \{0, 1\}^{m \times d}$.
    Then the cardinality of the embedded search space $\calS_{\dict}$ can be bounded above by
    \[
        \left | \calS_{\dict} \right|
        \leq \left[ (\mu_{\dict}+1)(d+1-\mu_{\dict}) \right]^{\lfloor m/2 \rfloor} (d+1)^{m \Mod 2}
    \]
    where $\mu_{\dict} = \max_{i,j} \ \max(h(\atom_i, \atom_j), h(\neg \atom_i, \atom_j))$,
    and $h$ is the Hamming distance.
\end{appprop}
\begin{proof}
First, we consider one anchor point.
Let $d \in \N$, and $\atom \in \bin^d$. Then for any $\bz \in \bin^d$, $\hamm(\atom, \bz) \in \N$
and
\[
0 \leq \hamm_\atom(\bz) = h(\atom, \bz) = \sum_i \delta(a_i, z_i) \leq d,
\]
so $\hamm_\atom(\bz) \in [d]$ and $|\calS| = d+1$.
Na\"ively generalizing this to $n$ dimensions would yield $|\calS| \leq (d+1)^n$.
However, the true cardinality is much lower, because having certain elements in common with one anchor point will restrict the corresponding dimensions to be the same with another anchor point. The next paragraph will make this intuition precise.

Next, we consider two anchor points.
Let $d \in \N$, and $\atom_1, \atom_2 \in \bin^d$. Then for any $\bz \in \bin^d$,
Suppose $\dict = [\atom_1, \atom_2]$, and let
\[
s = \{i \in [d] \ | \ [\atom_1]_i = [\atom_2]_i\}
\]
be the set of indices for which the anchors have take the same values,
and $\neg s = [d] \backslash s$, $|s| = k$.
Then we can express the embedding as
\[
\begin{aligned}
\hamm_\dict(\bz)
&= \hamm_{\dict_s}(\bz_s) + \hamm_{\dict_{\neg s}}(\bz_{\neg s}) \\
&= [h(\atom_{1, s}, \bz_{s}), h(\atom_{2, s}, \bz_{s})] + [h(\atom_{1, \neg s}, \bz_{\neg s}), h(\atom_{2, \neg s}, \bz_{\neg s})] \\
&= [h(\atom_{1, s}, \bz_{s}), h(\atom_{1, s}, \bz_{s})] + [h(\atom_{1, \neg s}, \bz_{\neg s}), h(\neg \atom_{1, \neg s}, \bz_{\neg s})] \\
&= [z_s, z_s] + [h(\atom_{1, \neg s}, \bz_{\neg s}), (d - h(\atom_{1, \neg s}, \bz_{\neg s}))] \\
&= [z_s, z_s] + [z_{\neg s}, (d - z_{\neg s})], \\
\end{aligned}
\]
where $z_s = h(\atom_{1, s}, \bz_s)$.
Now, $n_s = h(\atom_{s}, \bz_{s}) \in [k]$ and $n_{\neg s} \in [d-k]$.
The cardinality of the embedding space is exactly $(k+1) (d+1-k)$,
because a subset of $d-k$ variables always take the same values in both dimensions,
and the remaining $k$ move linearly independently to the first.
Differentiating the cardinality with respect to $k$:
\[
\begin{aligned}
\frac{d}{dk} (k+1) (d+1-k)
&= d -2k \leq 0 \qquad \text{for} \qquad \lceil d/2 \rceil \leq k \leq d,
\end{aligned}
\]
we see that the cardinality is an even symmetric function around $k = \lceil d/2 \rceil$, where it achieves its maximum.
This inspires the definition of the coherence-like quantity $\mu_\bA$,
whose value is monotonically related to the cardinality equation above,
and satisfies $\lceil d/2 \rceil \leq \mu_\bA \leq d$.
Further, note that for two anchor points, $\mu_\bA =
\max(h(\neg \atom_1, \atom_2), h(\atom_1, \atom_2))
= \max(k, d-k)$.
For $m$ row, $\mu_\bA$ is an upper bound on any pairwise similarity between all rows and their negations.
Therefore, we can apply the bound above to $\lfloor m/2 \rfloor$ pairs and have at most $(d+1)$ more values from the remaining dimension if $m$ is odd.
\end{proof}

We are now ready to combine our analysis of the cardinality of the embedded search space $\calS_\bA$ with the general result of
\cite{ucb} to get an improved regret bound for \ourmethod{}.
We will use $\mathcal{O}^*$ to denote $\mathcal{O}$ with log-factors suppressed.
\begin{appthm}
\label{appthm:regret}
  Let $\dict$ have $m$ rows, $\delta \in (0, 1)$, and $\beta_t = 2 \log(|\calS_\dict|t^2\pi^2/6\delta)$.
  Then the cumulative regret associated with running UCB for a sample $f$ of a zero-mean GP with kernel function $k_\ourmethod{}(\bz, \bz') = k_{\text{base}}(\hamm_\dict(\bz), \hamm_\dict(\bz'))$, is upper-bounded by $\mathcal{O}^{*}(\sqrt{T \gamma_T m})$ with probability $1 - \delta$,
  where $\gamma_T$ is the maximum information gain of $k_{\text{base}}$.
\end{appthm}
\begin{proof}
\ourmethod{} is equivalent to running canonical Bayesian optimization with the $k_{\text{base}}$ kernel
on the transformed search space $\calS_{\dict} = \bigl\{ \hamm_\dict(\bz) \ | \ \bz \in \{0, 1\}^d\bigr\}$.
Since $\calS_\dict$ is finite, Theorem~1 of \cite{ucb} applies,
with $\calS_\dict$ as the search space and
the information gain $\gamma_T$ of the base kernel $k_{\text{base}}$,
giving us a regret bound of $\mathcal{O}^*(\sqrt{T \gamma_T \log(|\calS_\dict}|))$.
Applying the cardinality bound of Proposition~\ref{apppropr:cardinality},
we get $\mathcal{O}^*(\log(|\calS_\dict|))
= \mathcal{O}^*(\log([(\mu_\bA + 1)
(d + 1 - \mu_\bA)]^{\lfloor m/2 \rfloor})
=
\mathcal{O}^*(m)$.
Plugging this cardinality bound into the generic asymptotic bound finishes the proof.
\end{proof}

\section{Dictionary Construction Approach via Binary Wavelets}
\label{app:sec:BinaryWavelet}

In this section, we describe a randomized dictionary construction approach based on Binary wavelet transform for binary spaces $\mathcal{Z}$=$\{0,1\}^d$.
At a high-level, this approach has two key steps.
First, we employ a deterministic recursive procedure to construct a pool of basis vectors over binary structures. Second, we randomly select a subset of $k$ diverse vectors as our dictionary $\dict$.
We explain the details of these two steps below.

\paragraph{Recursive algorithm for binary wavelet design.} The effectiveness of surrogate model critically depends on the dictionary employed to embed the discrete inputs.
We define our dictionary matrix  $\dict_{[k \times d]}$ as a subsampled ($k$-sized) set of basis vectors over the binary space $\{0,1\}^d$ which is characterized by the constituent vectors varying over a range of sequencies.
The notion of \emph{sequency} is defined as the number of changes from 1 to 0 and vice versa (analogous to the notion of frequency in Fourier transforms).

Multi-resolution {\em wavelets} \citep{mallat1989theory} are effective well-known techniques for studying real-valued signals at different scales by applying a set of orthogonal transforms to the data.
Specifically, binary wavelet transforms \citep{swanson1996binary} allow us to study data defined over binary spaces (concretely $\{0, 1\}^d$ with {\em mod 2} arithmetic) at different scales. Hence, they are a natural choice for constructing our pool of basis vectors.

We construct the randomized dictionary $\dict$ by randomly sampling from a deterministic binary wavelet transform matrix $\bold B_d$ generated by a recursive procedure as described in~\citep{swanson1996binary} (where such matrices were used for image compression).
The key idea behind the procedure is to recursively generate binary matrices whose vectors are ordered in terms of increasing sequency. Algorithm \ref{alg:bin_wavelet} provides the pseudo-code of this recursive method.

Given $\bold B_d$, the dictionary $\dict$ is constructed by subsampling row vectors from $\bold B_d$ i.e. $\dict=\bold P \bold B_d$ where $\bold P$ randomly samples $m$ vectors uniformly.
The random sampling using $\bold P$ picks vectors that are spread over a range of sequencies in contrast to the alternative choice of picking top-$m$ rows from $\bold B_d$ which restricts the chosen vectors to limited range of sequencies.
Our experiments demonstrate the effectiveness of randomized dictionaries over the top-$m$ alternative.

{\bf Remark.} Following Proposition~\ref{prop:affine}, this choice of dictionary is equivalent to the Subsampled Randomized Hadamard Transform (SRHT) for constructing low-dimensional embeddings in {\em continuous input spaces} where the embeddings are subsampled projections of Hadamard transforms, i.e., $\bf{\hat{x}} = \bold P \bold H_n \bold D \bx$ where $\bold H_n$  is the Hadamard matrix of order $n$, $\bold D$ is a diagonal matrix with random entries on the diagonal from $\{1,-1\}$ and $\bold P$ defined similarly as above. Importantly, this dictionary also minimizes coherence-type measure $\mu_\bA$ introduced in proposition 3.

\begin{figure*}[!ht]
    \centering
    \subfloat{
    \includegraphics[width=0.4\textwidth]{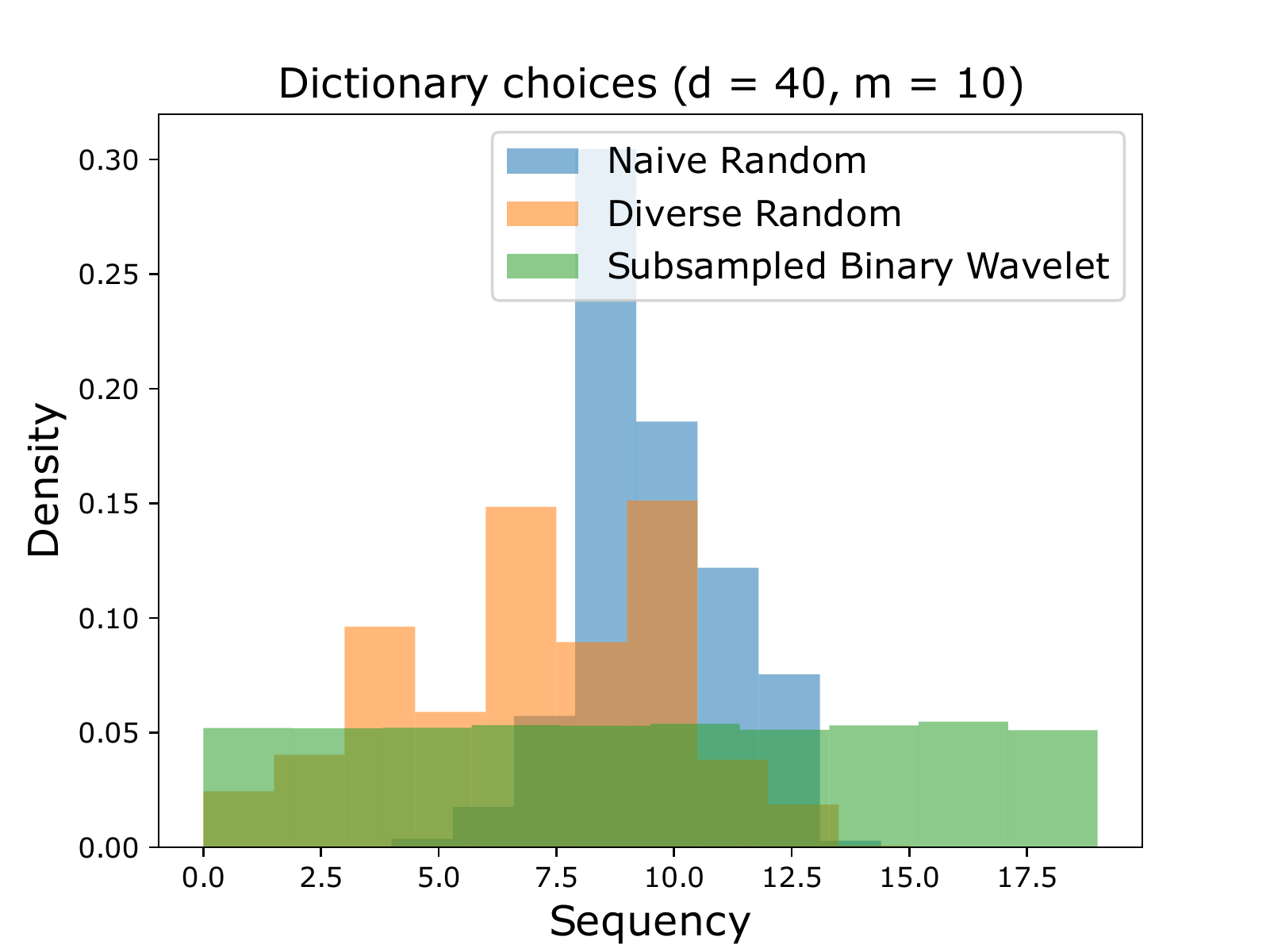}
    \label{fig:qap}
    }\quad
    \subfloat{
    \includegraphics[width=0.4\textwidth]{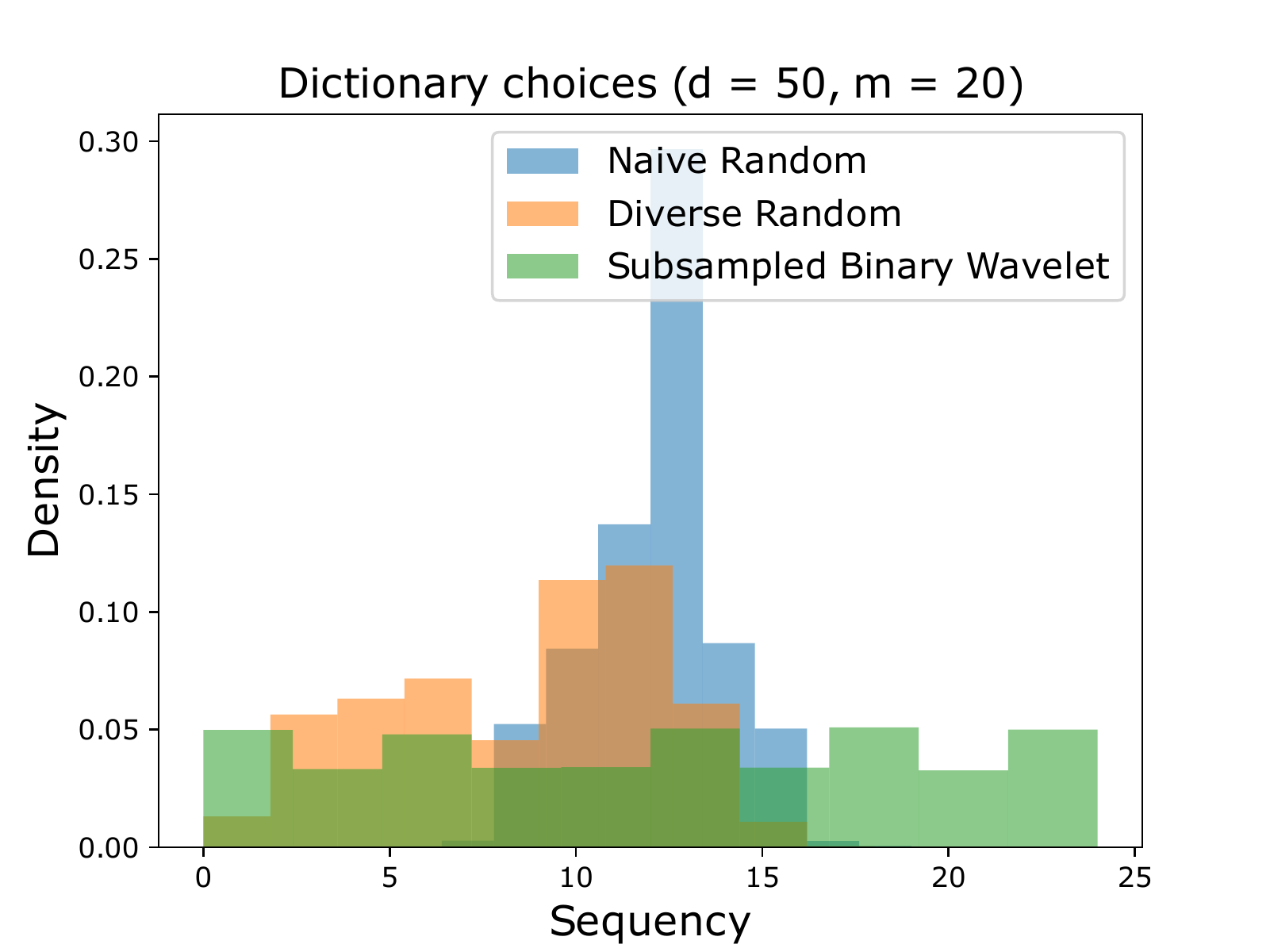}
    \label{fig:pl}
    }
    \caption{
        We randomly generated a large number ($1000$) of dictionaries.
        This shows the similarity of the two dictionary choices (binary wavelet) and (diverse parameter) in terms of the sequency characterization. The notion of sequency allows us to empirically see the similarity between these two better choices of the dictionary.
    }
    \label{fig:hist_densities}
\end{figure*}

\begin{figure*}[!ht]
    \centering
    \subfloat[LABS ($50$ binary parameters)]{
    \includegraphics[width=0.455\textwidth]{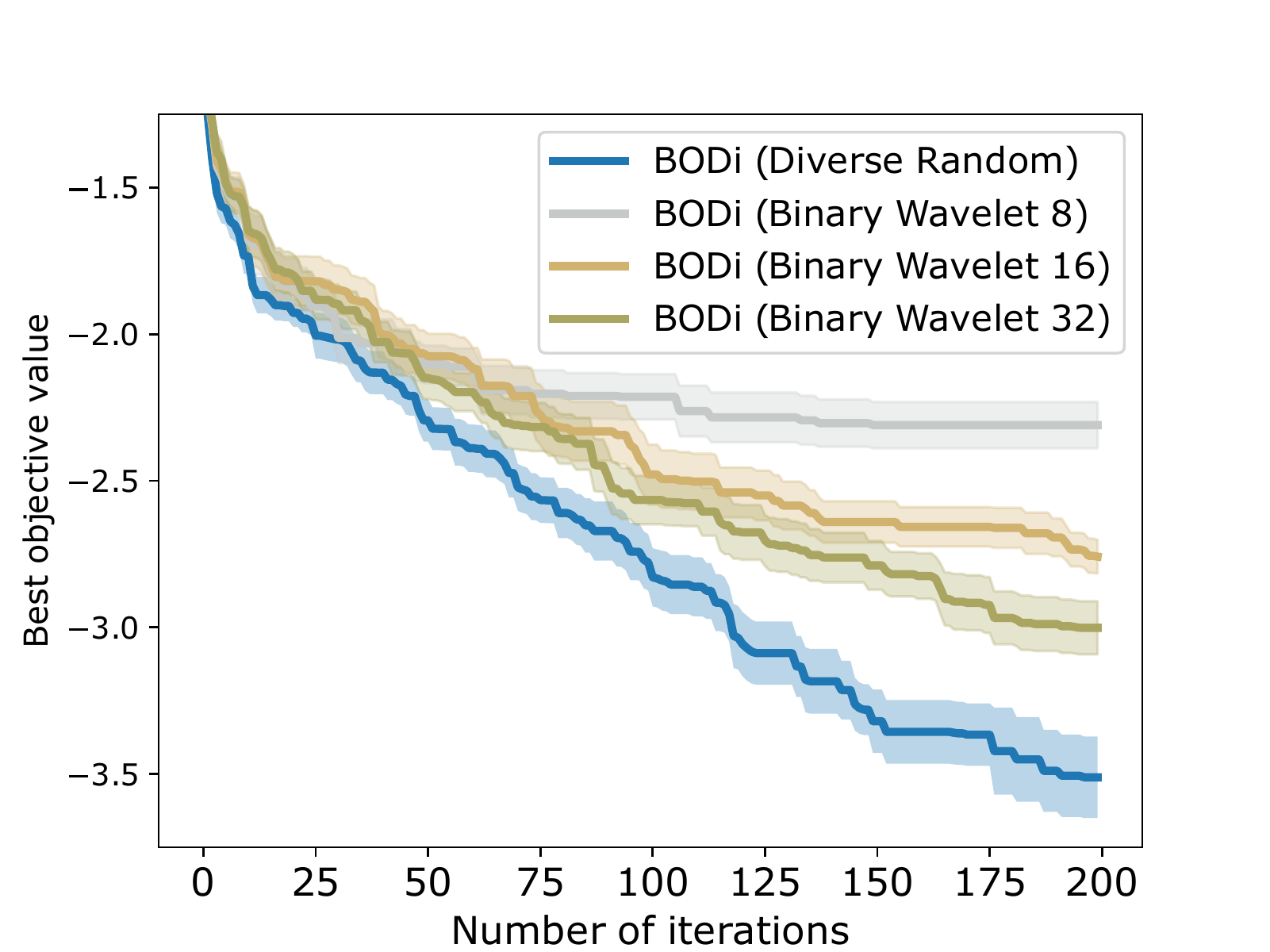}
    \label{fig:binwavelet_labs}
    }\quad
    \subfloat[MaxSAT ($60$ binary parameters)]{
    \includegraphics[width=0.450\textwidth]{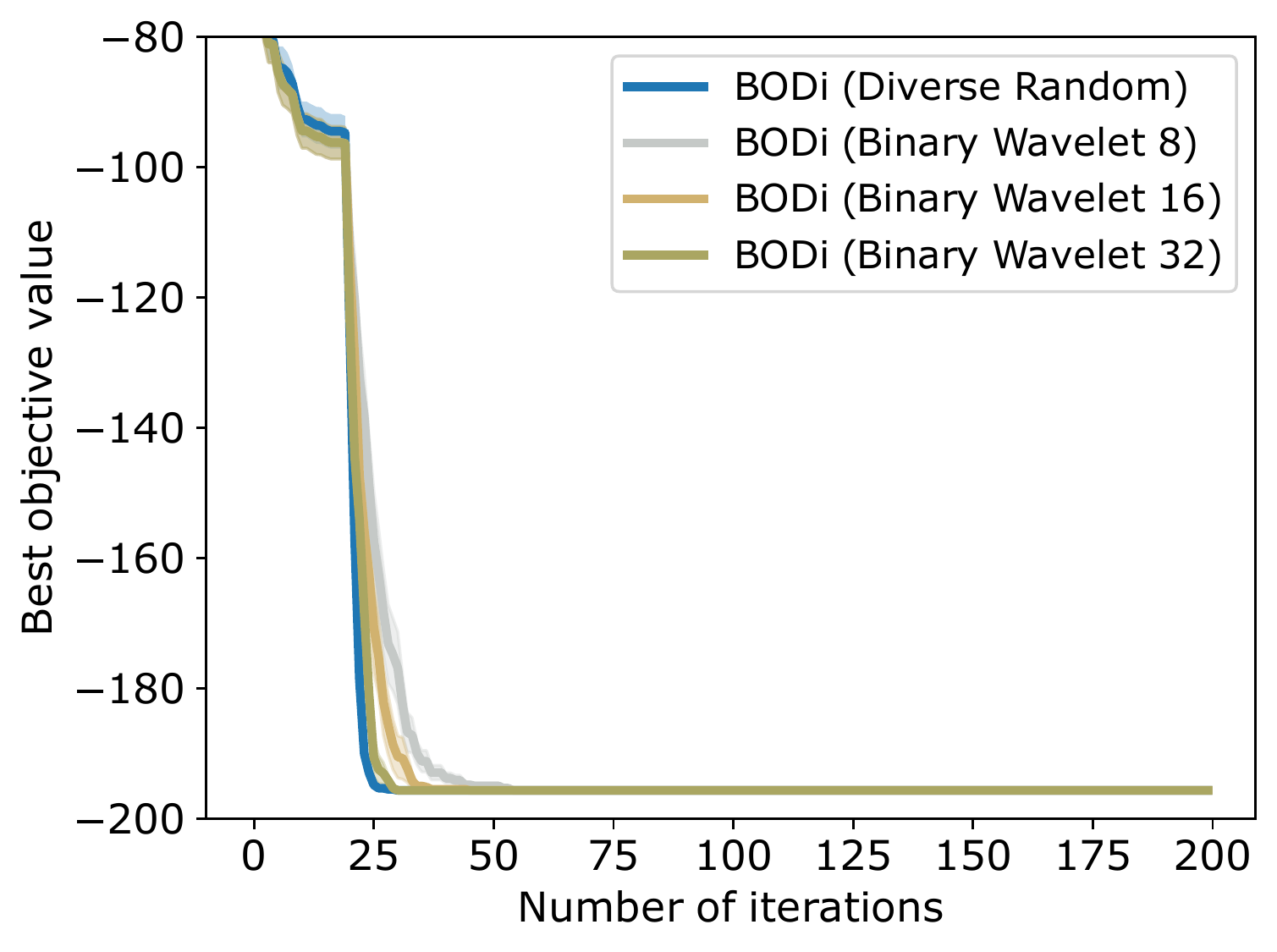}
    \label{fig:binwavelet_maxsat60}
    }\quad
    \subfloat[Ackley ($50$ binary \\ parameters, 3 continuous parameters)]{
    \includegraphics[width=0.45\textwidth]{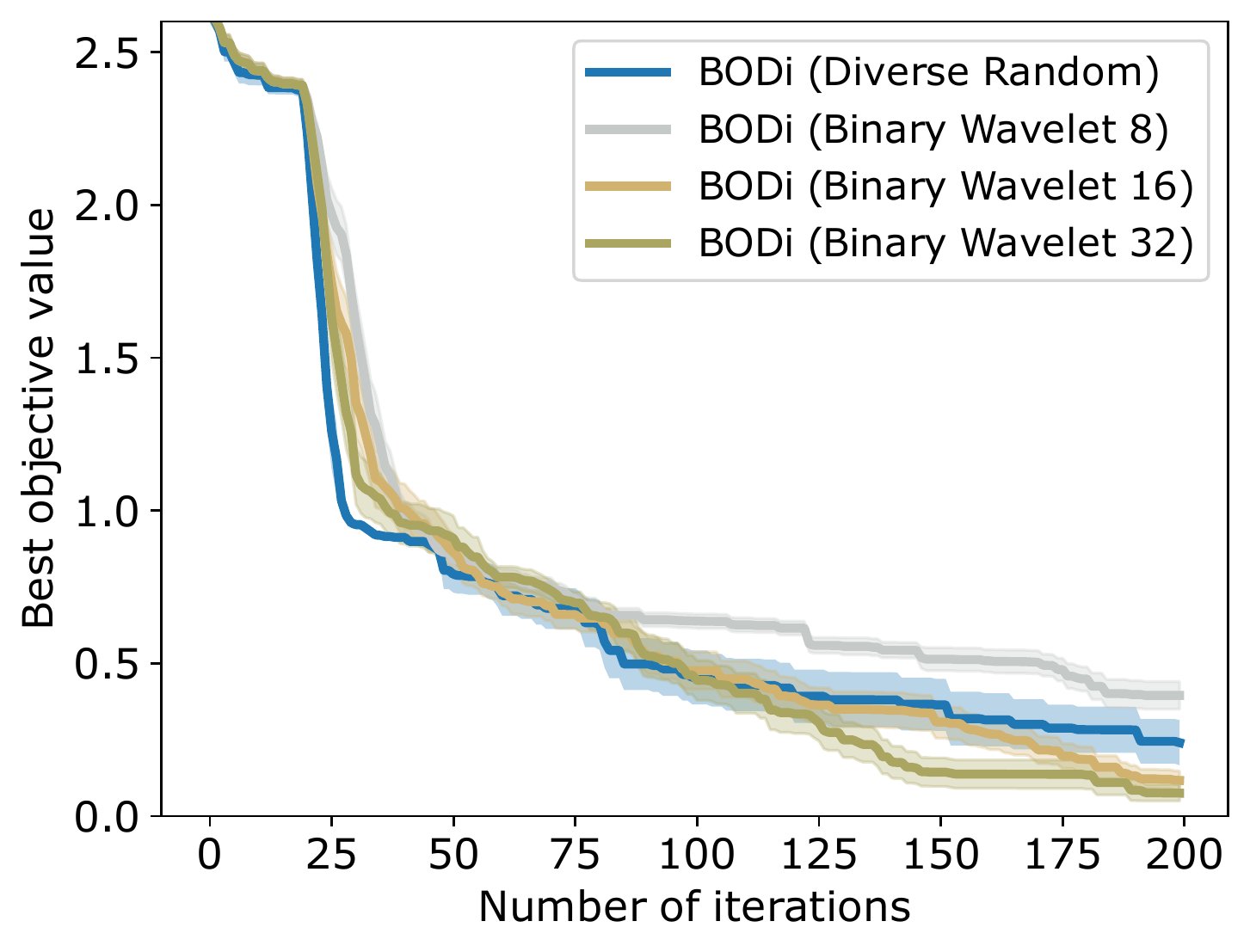}
    \label{fig:binwavelet_ackley}
    }\quad
    \subfloat[SVM ($50$ binary \\ parameters, 3 continuous parameters)]{
    \includegraphics[width=0.45\textwidth]{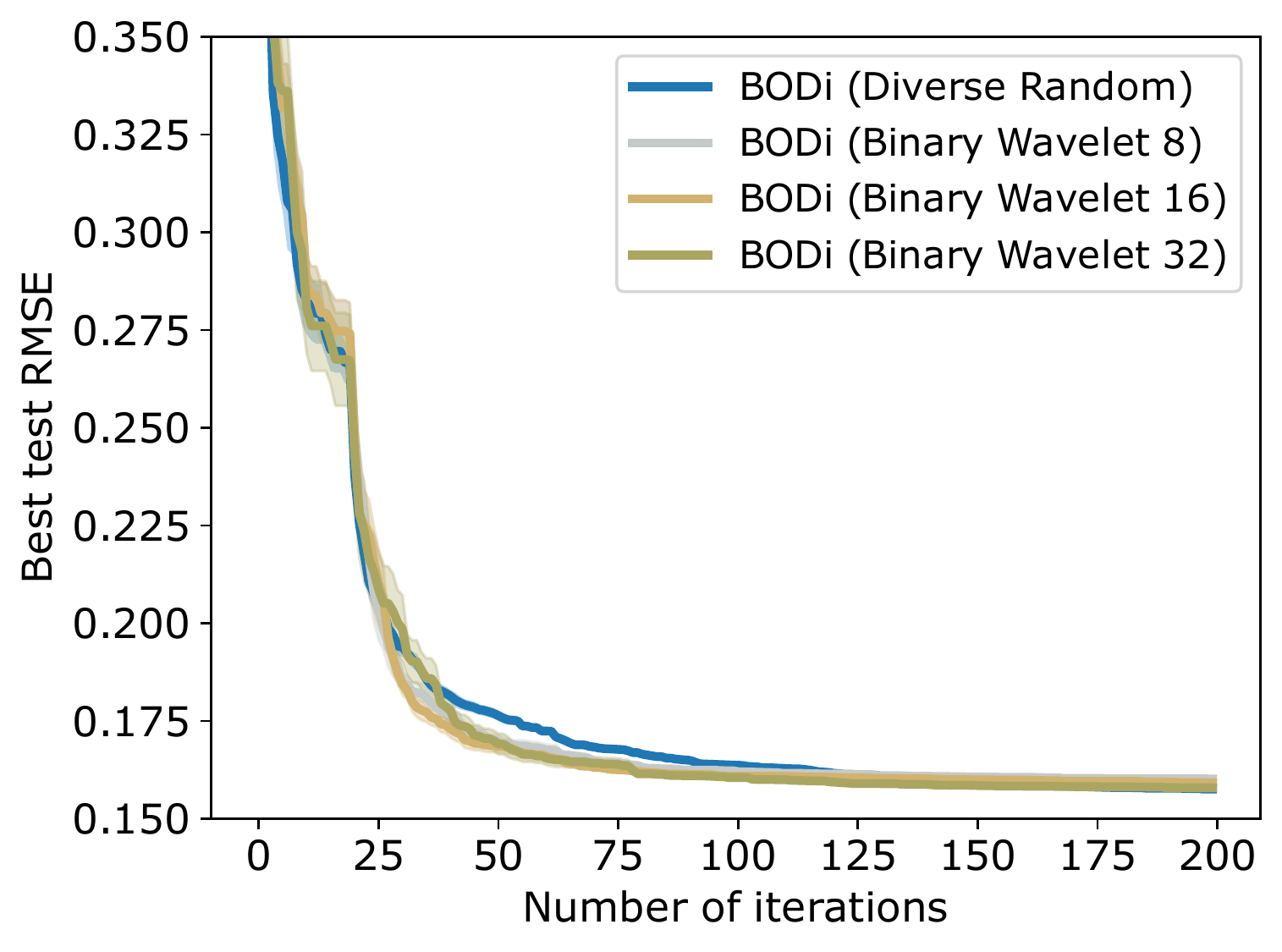}
    \label{fig:binwavelet_svm}
    }\quad
    \caption{Results comparing the two dictionary construction choices for \ourmethod{} (i.e., diverse random and binary wavelet). Overall, we find that binary wavelet design performs reasonably well but diverse random is a more robust choice considering all the benchmarks. Moreover, the diverse random choice can also be employed for categorical parameters unlike the binary wavelet construction which is limited to binary parameters.}
    \label{fig:binwavelet_experiments}
\end{figure*}

\newpage

\begin{algorithm}[!ht]
    \caption{\textsc{Binary Wavelet} ($n$) Transform}
    \textbf{requires}: input dimension $n$
    \begin{algorithmic}[1]
    \STATE if $n$ == $2$: \textbf{return} $\begin{bmatrix}
    1 & 1 \\
    1 & 0
    \end{bmatrix}$
    \STATE if $n$ == $4$: \textbf{return} $\begin{bmatrix}
    1 & 1 & 1 & 1\\
    1 & 0 & 0 & 0 \\
    1 & 0 & 1 & 1 \\
    1 & 0 & 1 & 0
    \end{bmatrix}$
    \STATE $B_{n-4}$= \textsc{Binary Wavelet} ($n$-4)
    \STATE Compute upper left $n-2 \times n-2$ matrix $\Gamma$
    \STATEx \hskip 1.0em $\Gamma = \begin{bmatrix}
        \mathbf{1}_{[2, 2]} & \mathbf{1}_{[2, n-4]} \\
        \mathbf{1}_{[n-4, 2]} & \neg B_{n-4} \\
    \end{bmatrix}$
    \STATE Set lower left block $\Delta^T \leftarrow \begin{bmatrix}1 & 0 & 1 & \cdots  \\
        1 & 0 & 1 & \cdots \\ \end{bmatrix}$
    \STATE Set lower right block $\Lambda \leftarrow \begin{bmatrix}1 & 1   \\
        1 & 0  \\ \end{bmatrix} $
    \STATE \textbf{return} $B_n = \begin{bmatrix}
    \Gamma & \Delta \\
    \Delta^T & \Lambda \\
    \end{bmatrix}$
    \end{algorithmic}
    {\label{alg:bin_wavelet}}
\end{algorithm}

\section{Local Search for Optimizing Acquisition Function over Combinatorial Spaces}
\label{ref:local-search-appendix}
In each iteration of optimizing the acquisition function, we first generate a set of initial inputs as starting points for local search over combinatorial inputs. These initial inputs are constructed by picking top-ranked candidates from a combined set of uniformly generated random inputs and spray inputs (Hamming distance based neighbors of incumbent best uncovered inputs of BO run). From each starting input, we run a greedy hill-climbing search where we move to the one-Hamming distance neighbor with the highest acquisition function value till convergence of the search or a maximum of $n_{ls}$ iterations. The best candidate among all the local search trajectories is picked as the next input for evaluation.

\FloatBarrier
\section{Runtimes}
Fig.~\ref{fig:avg_runtime} shows a runtime comparison of \ourmethod{}, \combo{}, and \casmo{}.
We show the average time to both fit the model and generate a new candidate on the MaxSAT problem with $60$ binary parameters.
\ourmethod{} uses the default of $128$ anchors which is used in all experiments.
We observe that \ourmethod{} and \casmo{} are significantly faster than \combo{} and on average take less than $10$ seconds per BO iteration.

\begin{figure*}[!ht]
    \centering
    \subfloat{
    \includegraphics[width=0.4\textwidth]{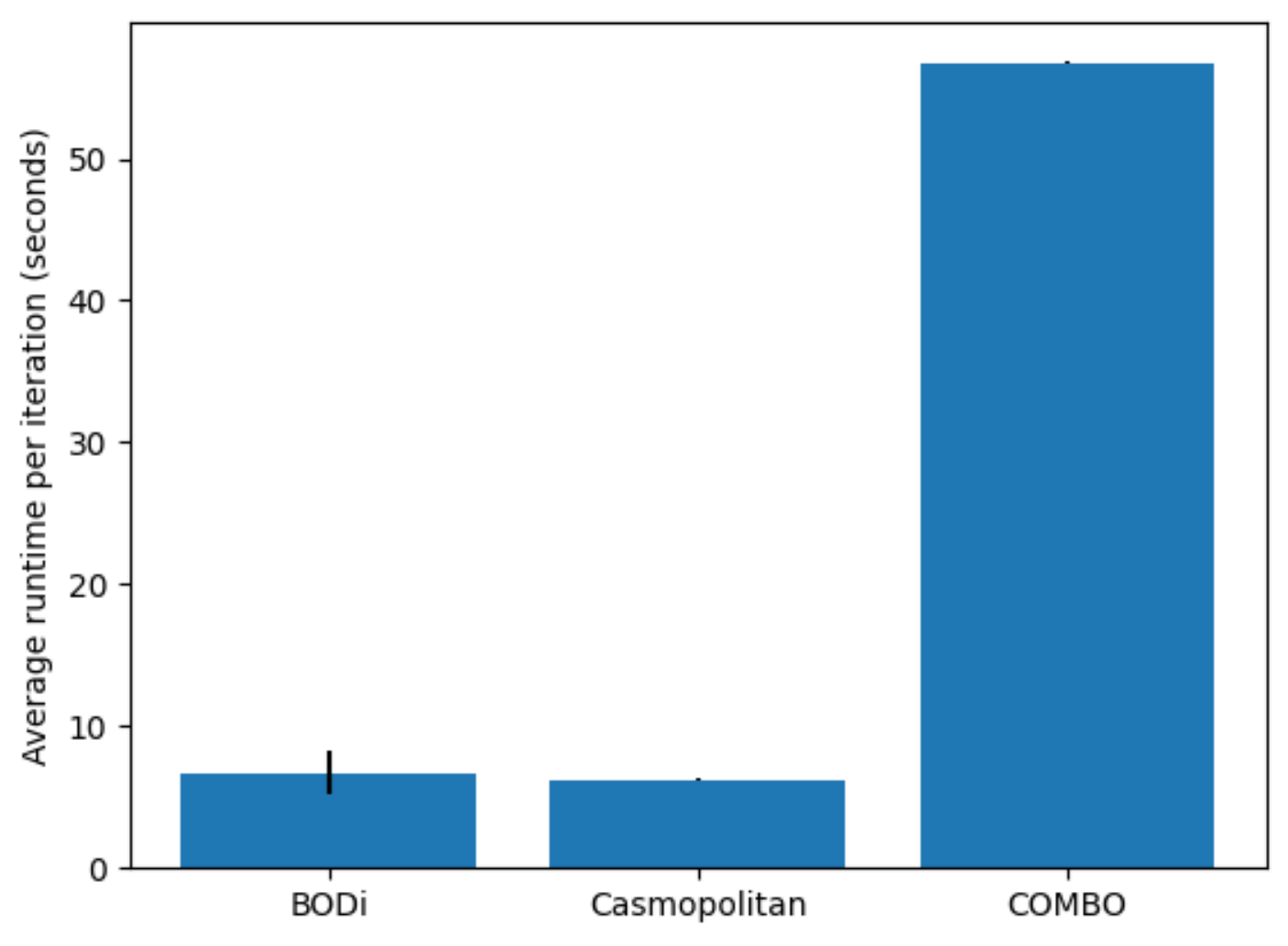}

    }
    \caption{
        Average runtime per iteration (in seconds) comparing \ourmethod{} with \casmo{} and \combo{}.
    }
    \label{fig:avg_runtime}
\end{figure*}

\section{Societal Impact}
Bayesian optimization is a commonly used approach for black-box optimization across broad variety of applications, including e.g. automated machine learning (AutoML). The primary benefit of our proposed \ourmethod{} method is better optimization performance for Bayesian optimization over combinatorial and mixed search spaces -- in the context of AutoML this would mean finding better models or finding similarly good models while using much less computational resources. We believe that such improvements pose minimal risk beyond more general concerns about potential misuse of the underlying application.

\section{Dictionary construction for dictionaries with binary and categorical variables \label{sec:categoricaldictionary}}

Algorithm \ref{alg:ps_dict2} provides pseudo-code for constructing dictionaries defined over binary input spaces $\{0, 1\}^d$. The key idea is to diversify the constructed dictionary by generating binary vectors determined by different bias parameters ($\theta$) of the Bernoulli distribution (unlike the naive random where $\theta$ is always $1/2$). This algorithm is generalized to the varying-sized categorical inputs in the following way (Algorithm \ref{alg:ps_dictcat}): for each of the $m$ elements of the dictionary $\dict$, we first sample a weight vector $\theta$ from the $\tau_{max}$-simplex $\Delta^{\tau_{max}}$, where $\tau_{max} = \max_j \tau_j$.  For each variable $v_j$, we then sample $\tau_j$ elements from $\theta$ and use those as the weight vector of a categorical distribution from which we in turn draw the $j$-th dimension of the dictionary element.

\begin{algorithm}[!ht]
    \caption{Dictionary design for binary input space  $\{0,1\}^d$ with diversely sparse rows
    }
    \textbf{requires}: dictionary size $m$
    \begin{algorithmic}[1]
        \STATE Dictionary $\dict \leftarrow $ empty
        \FOR{$i$=$1, 2, \ldots, m$}
            \STATE $\atom_i \leftarrow $ empty
            \STATE Sample Bernoulli parameter $\theta \sim \text{Uniform}(0, 1)$
            \FOR{$j$=$1, 2, \ldots, d$}
                \STATE Sample binary number $a \sim \text{Bernoulli}(\theta)$
                \STATE $\atom_i \leftarrow \atom_i \cup a$
            \ENDFOR
            \STATE Add $\atom_i$ to dictionary: $\dict \leftarrow \dict \cup \atom_i $
        \ENDFOR
        \STATE \textbf{return} the dictionary $\dict$ of size $m \times d$
    \end{algorithmic}
    {\label{alg:ps_dict2}}
\end{algorithm}

\begin{algorithm}[!ht]
    \caption{Dictionary design for discrete spaces with categorical variables via diverse parameters}
    \textbf{Input}: candidate sets $C(v_1),\dotsc, C(v_d)$, dictionary size $m$
    \textbf{Output}: the dictionary $\dict$ of size $m \times d$
    \begin{algorithmic}[1]
        \STATE Dictionary $\dict \leftarrow $ empty
        \STATE $\tau_{max} \leftarrow \max_j \tau_j$
        \FOR{$i$=$1, 2, \ldots, m$}
            \STATE $\atom_i \leftarrow $ empty
            \STATE Sample $\btheta \sim \Delta^{\tau_{max}}$
            \FOR{$j$=$1, 2, \ldots, d$}
                \STATE $\btheta_j \leftarrow $ sample (w/o repl.) $\tau_j$ elements from $\btheta$
                \STATE $\btheta_j \xleftarrow{} \btheta_j / \|\btheta_j\|_1$ (Normalize to yield distribution)
                \STATE $a \leftarrow $ sample from $C(v_j)$ with probabilities $\btheta_j$
                \STATE $\atom_i \leftarrow \atom_i \cup a$
            \ENDFOR
            \STATE Add $\atom_i$ to dictionary: $\dict \leftarrow \dict \cup \atom_i $
        \ENDFOR
    \end{algorithmic}
    {\label{alg:ps_dictcat}}
\end{algorithm}

{\noindent \bf Illustration of Algorithm \ref{alg:ps_dictcat}}

We illustrate the description in Algorithm \ref{alg:ps_dictcat} with a simple example for an input with the same number of candidate choices for each input dimension. Let's say, we want to construct a dictionary vector for  an input space with 10 variables (i.e. 10-dimensional input) where each dimension can take 4 values (i.e., $\tau_1 = \tau_2 …. = \tau_d = 4$). For each input dimension (for loop in line 6), we sample a value from a categorical distribution which is parameterized by weights $\theta_j$. In the naive random case, $\theta_j$ is [¼, ¼, ¼, ¼ ]. In contrast, algorithm \ref{alg:ps_dictcat} diversifies this vector $\theta_j$ by sampling from a simplex in line 5. The  variable $\tau_{max}$ (Line 2) and resampling of $\theta$ in Line 7 allows us to generalize the algorithm for the case where each input dimension can take a different number of values.